%% file: main.tex
\documentclass[preprint,12pt]{elsarticle}

\usepackage[T1]{fontenc}
\usepackage{graphicx}

\usepackage{amssymb}
\usepackage{amsmath}

\usepackage{amsthm}

\theoremstyle{definition}
\newtheorem{definition}{Definition}[section]

\theoremstyle{plain}
\newtheorem{theorem}[definition]{Theorem}
\newtheorem{example}[definition]{Example}

\newtheorem{lemma}[definition]{Lemma}

\usepackage{latexsym}
\usepackage{amssymb}
\usepackage{mathtools}
\usepackage{paralist}
\usepackage{enumitem}
\usepackage{color}
\usepackage{caption}
\usepackage{subcaption}
\usepackage{algorithm}
\usepackage{algpseudocode}
\usepackage{makecell}

\usepackage{etoolbox}

\DeclarePairedDelimiter\abs{\lvert}{\rvert}%
\DeclarePairedDelimiter\norm{\lVert}{\rVert}%
\makeatletter
\let\oldabs\abs
\def\abs{\@ifstar{\oldabs}{\oldabs*}}
\let\oldnorm\norm
\def\norm{\@ifstar{\oldnorm}{\oldnorm*}}
\makeatother

\usepackage{tikz}
\usetikzlibrary{positioning}
\usetikzlibrary{arrows}
\usetikzlibrary{tikzmark}
\usetikzlibrary{calc}
\usetikzlibrary{shapes.geometric}

\tikzset{
node/.style={circle, draw=black},
block/.style={draw=black}
}
\tikzset{every picture/.style={line width=0.6pt}}

\usetikzlibrary{arrows.meta,positioning,shapes,fit,bending,quotes}

\tikzset{
  diedge/.style = {
   semithick, ->, >={[round,sep,bend]Stealth}
  }
}

\tikzset{
  biedge/.style = {
   semithick, <->, >={[round,sep,bend]Stealth}
  }
}

\newread\hlfile
\newcommand{\inputitemized}[1]{%
  \begingroup
    \openin\hlfile=#1
    \read\hlfile to\linehl
    \loop\unless\ifeof\hlfile
      \item \linehl
      \read\hlfile to\linehl
    \repeat
    \closein\hlfile
  \endgroup
}

\journal{International Journal of Approximate Reasoning}

\begin{document}

\begin{frontmatter}

\title{Denoising the Future: Top-p Distributions for Moving Through Time\tnoteref{t1}} %

\author[aff1]{Florian Andreas Marwitz\corref{cor1}}
\ead{florian.marwitz@uni-hamburg.de}

\author[aff1]{Ralf Möller}

\author[aff2]{Magnus Bender}

\author[aff1]{Marcel Gehrke}

\affiliation[aff1]{organization={Institute of Humanities-Centered Artificial Intelligence, University of Hamburg},
            addressline={Warburgstra{\ss}e 28}, 
            city={Hamburg},
            postcode={20354},
            country={Germany}}

\affiliation[aff2]{organization={Department of Management, Aarhus University},
            addressline={Universitetsbyen 61},
            city={Aarhus C},
            postcode={8000},
            country={Denmark}}

\cortext[cor1]{Corresponding author}
\tnotetext[t1]{The research for this paper was funded by the Deutsche Forschungsgemeinschaft (DFG, German Research Foundation) under Germany's Excellence Strategy – EXC 2176 'Understanding Written Artefacts: Material, Interaction and Transmission in Manuscript Cultures', project no. 390893796. The research was conducted within the scope of the Centre for the Study of Manuscript Cultures (CSMC) at Universität Hamburg.}

\begin{abstract}
Inference in dynamic probabilistic models is a complex task involving expensive operations.
In particular, for Hidden Markov Models, the whole state space has to be enumerated for advancing in time.
Even states with negligible probabilities are considered, resulting in computational inefficiency and possibly increased noise due to the propagation of unlikely probability mass.
We propose to denoise the future and speed up inference by using only the top-$p$ transitions, i.e., the most probable transitions with accumulated probability $p$.
We show that the error introduced by using only the top-$p$ transitions is bound by $p$ and the so-called minimal mixing rate of the underlying model.
We also show the same bound when using only the top-$p$ states, which is the same, just for the states.
Moreover, in our empirical evaluation, we show that we can, when using top-$p$ transitions, expect speedups of at least an order of magnitude, while the error in terms of total variation distance is below $0.09$.
Using the top-$p$ states is slower than top-$p$ transitions since we iterate over all states in each time step and sometimes lead empirically to a higher error.
With a more sophisticated implementation, the speed-up, if any, would be really small.
While top-$p$ transitions look really promising, we cannot recommend top-$p$ states and discuss why it is of the slower, while the error does not necessarily decrease.
\end{abstract}


\end{frontmatter}

\input{introduction}

\input{preliminaries}

\input{top_p}

\input{error}

\input{evaluation}

\input{conclusion}

\bibliographystyle{elsarticle-num-names} 
\bibliography{literature}

\end{document}

%% file: introduction.tex
\section{Introduction}
Inference in dynamic probabilistic graphical models (dynamic PGMs)~\cite{koller2009probabilistic} is in general NP-hard: The probability distribution over all states has to be advanced in time.
However, certain future states are highly unlikely.
Our goal is to speed up the inference time and denoise the future by only considering events with a high probability mass.
To briefly illustrate our idea: Suppose we have a simple weather model and the sun is shining.
In the next time step, the conditions partly cloudy and light rain are much more probable than, e.g., heavy thunderstorm, which we then do not want to consider.
In this paper, we propose to use only the \emph{top-$p$} events, i.e., the most probable events with accumulated probability of at least $p$, for inference, where we only iterate about the more likely states, for inference in Hidden Markov Models (HMMs) as a representative for PGMs.
This way, we reduce the runtime while retaining control over the induced error by setting $p$ appropriately.
We can denoise the future either \textit{offline}, i.e., adapt the model parameters to only consider a certain probability mass over the transitions and thereby making the model sparse to proceed in time, in an offline preprocessing step, or \textit{online}, i.e., advance only events of the current state with a certain probability mass, to obtain a reduced state space, while accounting for observations.
We call the offline approach HMM$_{off}$ and the online approach HMM$_{on}$.
These approaches can also be viewed as a generative AI approach, we generate one possible world, for which we have probability and error assertions.

Let us first focus on HMM$_{off}$ and start by shortly motivate the runtime savings with language models, e.g., Llama 3~\cite{llama3}.
In language models, we have multibillion parameters and more possible events to be generated than words in the English language.
Now, take for example the words \emph{The weather is}, then the probability of the next word being \emph{rainy} or \emph{sunny} is much more probable than \emph{eating}.
With our \emph{top-$p$} approach, we can thus spare the iteration over lots of unlikely states.
Let us dive a bit deeper into what we actually do by going back to our initial weather model:
When we have the weather condition \emph{sunny}, the next ones could be \emph{partly cloudy}, \emph{light rain}, and \emph{foggy} with relatively high probabilities while a direct transition to more extreme conditions like \emph{heavy rain} or \emph{thunderstorm} is less likely.
For inference, we skip the enumeration of \emph{heavy rain} and \emph{thunderstorm}, which is what we call denoising the future.
While this change may seem marginal for a minimalistic example at first, it is of great use for very large models:
First, the savings are present in each time step, accumulating and reducing the overall time.
Second, the model is made sparse, helping human understanding.
Also, the obtained sparse HMM$_{off}$ can be used in edge computing, where a full HMM may be too big to run.
Third, the error introduced is bounded and measurable and, therefore, controllable.

Let us now have a look at the runtime savings for HMM$_{on}$ with our weather example.
Assume, most of the probability mass of the current state is distributed over the events \emph{sunny}, \emph{partly cloudy}, and \emph{light rain} and only a negligible probability mass of the current state is distributed over the events \emph{foggy}, \emph{heavy rain}, and \emph{thunderstorm}.
Then, we can redistribute the probability mass from \emph{foggy}, \emph{heavy rain}, and \emph{thunderstorm} to \emph{sunny}, \emph{partly cloudy}, and \emph{light rain} and thereby set the probabilities for \emph{foggy}, \emph{heavy rain}, and \emph{thunderstorm} to $0$.
Having certain events with a probability of $0$ means that complete rows of the transition matrix can be skipped.
Another way of looking at the chance is to actually reduce the state space and completely drop the events that are not possible and accordingly reduce the dimensionality of the transition matrix. 
Being an online approach, the state has to be reduced for each time step, adding some runtime as we have to iterate over all states in each time step to filter out negligible ones, but can save quite some time, similar to the HMM$_{off}$ approach, while proceeding in time.
The online approach has another advantage: It leaves the model untouched and only alters the current state distribution.
Therefore, accounting for observations may be easier, as HMM$_{off}$ can, in theory, run into a self-loop.
Also for this approach, the introduced error is bounded and measurable and, therefore, controllable.

In both of our approaches, probability masses of unlikely events are distributed, leading to potential problems regarding minorities.
We note here that the problems arising do not originate from our approach, but from the underlying model.
Though, our approach amplifies biases present in the model, e.g., not allowing the switch from sunny to thunderstorm.
We also highlight that in applications, where our approach can lead to problems regarding minorities, tools from the field of AI should not be used at all or at least with high caution, as they are prone to biases.
However, such an approach can also be analyzed under the lens of preserving privacy.

\paragraph{Contribution}
In this paper, we propose using only the top-$p$ events for inference in HMMs as a representative for PGMs.
To that end, we propose HMM$_{off}$, which applies the idea of top-$p$ events to the transition matrix, as well as HMM$_{on}$, which applies the idea of top-$p$ events to the current state.
Both approaches inherently have the idea to create sparsity and use that to reduce the number of computations.
Specifically, we show that by using the top-$p$ events:
\begin{inparaenum}[(i)]
    \item the theoretical error for both approaches is bounded in terms of total variation distance by $\frac{1-p}{\gamma}$, where $\gamma$ is a model-specific parameter,
    \item the sparsity of distributions can be significantly increased using top-$p$ events,
    \item the introduced error in our examples is negligible, e.g., only $0.09$ for the total variation distance for top-$0.9$ in terms of total variance for HMM$_{off}$, 
    \item the runtime for inference can be reduced for orders of magnitude for HMM$_{off}$.
\end{inparaenum}

This paper is an extended version of the paper submitted to the ECSQARU 2025~\cite{ecsqaru}.
While reworked text, the extensions include correcting the proofs of the ECSQARU version and describing and analyzing using the top-$p$ states (instead of transitions) as well as the combination of both in the evaluation.

\paragraph{Related Work}
Rabiner introduces HMMs~\cite{hmm} and proposes to use the Forward-Backward algorithm~\cite{baum1966statistical} for inference.
Zhang et al. distill a language model into an HMM~\cite{gelato}.
HMMs can be extended to Dynamic Bayesian Networks (DBNs)~\cite{dbn}, which can model interactions between different state variables.
Inference is done by, e.g, the interface algorithm~\cite{murphy_diss}.
However, query answering in Bayesian networks~\cite{pearl_intro_bayesian_networks} is NP-hard~\cite{bn_qa_np}, giving rise to approximations.
Boyen and Koller represent the distribution over the random variables with successors in the next time step as a product of marginals~\cite{error_bound}, which is extended by Murphy and Weiss to always use a factored representation~\cite{murphy2001factored}.
Instead of instantiating a DBN for each time step, Gao et al. introduce a sliding window approach~\cite{gao2014approximate}.
Murphy provides an overview over approximate inference~\cite{murphy_diss}.
Though, the focus in literature is on approximate inference in DBNs or approximation for learning HMMs, but not on approximate inference within HMMs, which we propose in this paper.
Vithanage et al.~\cite{vithanage2006approximate} effectively use a top-$k$ approach, i.e., the $k$ most probable transitions, and show that this minimizes the Kullback-Leibler divergence if restricting to $k$ transitions.
Contrary, we provide a more fine-grained error assertion.
Also, when fixing $p$, one indirectly fixes $k$, too.

\paragraph{Structure}
The remainder of this paper is structured as follows:
First, we introduce HMMs.
Second, we explain our top-$p$ approaches, HMM$_{off}$ and HMM$_{on}$, and how to use it.
Third, we provide a theoretical analysis of the error and bound the error.
Fourth, we provide an empirical evaluation including remarks on the implementation of both approaches.

%% file: preliminaries.tex
\section{Hidden Markov Models Generalizing Dynamic Models}
In this section, we lay the foundation for our top-$p$ approaches.
We introduce HMMs as a representative of dynamic PGMs.
An HMM consists of two random variables per time step: One for the state and one for the observation (c.f. Figure~\ref{figure:hmm}), e.g., the weather condition and whether we see a person wearing a raincoat.
The probability of an observation depends only on the current state and the probability of the next state depends only on the current state:

\begin{definition}[Hidden Markov Model~\cite{hmm}]
    A \emph{Hidden Markov Model} consists of two series of random variables $(S_t,O_t)_t$ over a set $S$ of states and $O$ of observations, where $S_t$ denotes the state random variable in time step $t$ and $O_t$ the respective observation random variable.
    The probability $P(O_t \mid S_t)$ of an observation depends only on the current state.
    The probability distribution over the states is given as the temporal behavior $P(S_t \mid S_{t-1})$ and a prior $P(S_0)$.
    The semantics for a time step $t$ is thus given by
    \begin{align}
        P(S_t \mid S_{t-1}) &= \sum_{s_{t-1} \in S_{t-1}} P(S_t \mid s_{t-1}) \cdot P(s_{t-1}) \quad \text{and} \label{eq:state_transition}\\
        P(O_t) &= \sum_{s_t \in S_t} P(O_t \mid s_t) \cdot P(s_t),
    \end{align}
    where $P(S_t)$ is the distribution specifying the probabilities $s_t$ for each state $s \in S$.
\end{definition}

Figure~\ref{figure:hmm} shows the graphical representation of an HMM.
We illustrate the definition of an HMM with our introductory example:

\begin{figure}
    \centering
    \begin{tikzpicture}[node/.style={circle, draw},
        baseline=(current bounding box.center)]
        
        \node[node] (s0) {$S_0$};
        \node[node, below=of s0] (o0) {$O_0$};
        
        \node[node, right=of s0] (s1) {$S_1$};
        \node[node, below=of s1] (o1) {$O_1$};
        
        \node[node, right=of s1] (s2) {$S_2$};
        \node[node, below=of s2] (o2) {$O_2$};

        \node[node, right=of s2] (s3) {$\dots$};
        
        \draw[->] (s0) -- (o0);
        \draw[->] (s1) -- (o1);
        \draw[->] (s2) -- (o2);

        \draw[->] (s0) -- (s1);
        \draw[->] (s1) -- (s2);
        \draw[->] (s2) -- (s3);
        
    \end{tikzpicture}
    \caption{Representation of an HMM as graph.}
    \label{figure:hmm}
\end{figure}

\begin{example}
    In each time step, the \emph{state} is the weather condition.
    The \emph{observation} is the probability of wearing a raincoat.
    Table~\ref{table:simple_weather} gives the state transition probabilities.
    The initial state distribution is uniform.
    The probability of wearing a raincoat is the sum over the \emph{rainy} conditions light rain, heavy rain, and thunderstorm.
    Thus, the probability of not wearing a raincoat is the sum over all other conditions, partly cloudy, foggy, and sunny.
    When it is sunny, the probability of partly cloudy in the next time step is $0.3$.
    The probability of wearing a raincoat when it is sunny is $0.35$.
\end{example}

\begin{table}[tbp]
    \centering
    \begin{tabular}{ccccccc}
        Next/current & \makecell{Partly\\cloudy} & \makecell{Light\\rain} & Foggy & Sunny & \makecell{Heavy\\rain} & Thunderstorm \\ \hline
        Partly cloudy & 0.3 & 0.2 & 0.3 & 0.3 & 0.1 & 0.1 \\
        Light rain & 0.2 & 0.2 & 0.2 & 0.25 & 0.2 & 0.2 \\
        Foggy & 0.1 & 0.1 & 0.2 & 0.15 & 0.1 & 0.1 \\
        Sunny & 0.2 & 0.1 & 0.1 & 0.2 & 0.1 & 0.1 \\
        Heavy rain & 0.1 & 0.2 & 0.1 & 0.06 & 0.2 & 0.3 \\
        Thunderstorm & 0.1 & 0.2 & 0.1 & 0.04 & 0.3 & 0.2
    \end{tabular}
    \caption{State transition probabilities giving the probability of transitioning to the next state (row) given the current state (column).}
    \label{table:simple_weather}
\end{table}

Now, the important part is that the sum in Equation~\ref{eq:state_transition} goes over all possible states, rendering the equation inefficient for large state spaces, like in language models.
Our goal is to reduce the number of states that have to be enumerated.
While an HMM seems simple, it can be thought of as a generalized dynamic PGM: We have a, potentially complex, state $S_t$ emitting observations $O_t$.
Let us denote the distributions $P(s_t), \,s_t \in S,$ as \emph{forward message} updating the state probabilities through time.
DBNs~\cite{dbn} can be viewed as HMMs by treating the joint probability distribution over the interface~\cite{murphy_diss} as the state and all other random variables jointly as the observation.
Furthermore, GPT-based language models can be viewed as HMMs and learning the forward message in a highly complex state space.
At their core, they can be viewed as HMMs, interpreting, analogous to DBNs, each possible complex internal state of the language model as a state of an HMM with respective transition and observation (or emission) probabilities.
Generating sentences from language models is then the same as predicting states and sampling observations from HMMs, which leads to a noisy distribution in the far future without any observations.

Throughout this paper, we use HMMs as a representative for dynamic PGMs.
The same ideas apply to Kalman filters~\cite{ai_modern} for continuos variables.
But for illustrative purposes, we show the idea for discrete variables in an HMM setting.
In the next section, we define the top-$p$ approach for HMMs to reduce the space that has to be enumerated for the forward message.

%% file: top_p.tex
\section{Inference in Top-$p$ Hidden Markov Models}
Our goal is to denoise the future and reduce the runtime for inference by using only the top-$p$ events, effectively reducing the number of enumerations by not accounting for values not included in the \emph{top-$p$ events}.
We introduce two approaches following the idea of sparse enumeration:
First, we focus on the highly probable transitions, denoising the less likely ones.
With HMMs as our probabilistic model, we present our HMM$_{off}$ approach:
Given an HMM, we alter the probabilities of the next state given the current state.
For all possible next states, we only keep those with the highest probability until we accumulated a total probability mass of $p$.
Second, we keep the transition model unchanged, and only use the top-$p$ states in HMM$_{on}$.
We first formalize the top-$p$ distribution for a single random variable and then extend it to HMM$_{off}$ and HMM$_{on}$.

Before we give the definition for the top-$p$ distribution, we give a helper definition to select the top probabilities:
\begin{definition}[Top-$p$ Events]
    Given a probability distribution $P$ with a set of events $X$, the \emph{sorted order} of $P$ is a sequence $Z \subseteq X$ with the events of $X$ arranged descending according to $P(X)$.
    The \emph{top-$p$ events} of $P$ are then the set $Y \subseteq X$ built from $Z$ by taking events until their probabilities reach~$p$.
\end{definition}

For events with equal probability, an arbitrary choice is made.
The top-$p$ events constitute the set of events whose probability will be kept and scaled later on, while the probability of events not in the top-$p$ set will be set to zero.
We illustrate the sorted order and top-$p$ events with our introductory example:
\begin{example} \label{example:top_p-set}
    Let $P$ denote the probability of the weather condition in the next time step given it is sunny in the current time step.
    Table~\ref{table:simple_weather} defines the probability distribution $P$.
    The sorted order of $P$ is then: Partly cloudy, light rain, sunny, foggy, heavy rain and thunderstorm.
    For $p=0.9$, the top-$p$ events consists of partly cloudy, light rain, foggy and sunny.
    For $p=0.91$, the top-$p$ events additionally includes heavy rain.
\end{example}

\begingroup
\renewcommand{\arraystretch}{1.1}
\begin{table}[tbp]
    \centering
    \begin{tabular}{cc}
        Weather condition & Probability\\ \hline
        Partly cloudy & $\frac{1}{3}$ \\
        Light rain & $\frac{5}{18}$ \\
        Foggy & $\frac{1}{6}$ \\
        Sunny & $\frac{2}{9}$ \\
        Heavy rain & 0 \\
        Thunderstorm & 0
    \end{tabular}
    \caption{Top-$p$ distribution for the state in the next time step given it is currently sunny for $p=0.9$.}
    \label{table:weather_sunny_top_p}
\end{table}
\endgroup

Building on the top-$p$ set, we define the top-$p$ distribution:
\begin{definition}[Top-$p$ Distribution] \label{definition:top_p_distribution}
    Given a probability distribution $P$ with a set of events $X$, the \emph{top-$p$ distribution} of $P$ is another probability distribution $top_p(P) \coloneq Q$ over the same set of events $X$.
    Let $Y$ be the top-$p$ set of $P$ for a given $p$.
    Then, $Q$ is defined as
    \begin{equation}
        Q(x) = \begin{cases}
            \frac{P(x)}{P(Y)} & x \in Y \\
            0 & \text{otherwise}
        \end{cases}\ .
    \end{equation}
\end{definition}

In other words, events with high probabilities are promoted while events with low probabilities are ignored, effectively amplifying the differences.
For Example~\ref{example:top_p-set} and $p=0.9$, Table~\ref{table:weather_sunny_top_p} shows the top-$p$ distribution. 
To obtain our top-$p$ transition function, we independently apply the top-$p$ distributions to them, as an HMM consists of conditional probability distributions:

\begin{definition}[HMM$_{off}$] \label{definition:top_p_hmm}
    Given an HMM $H$ with a set $S$ of states, a set $O$ of observations, probability distributions over the states given by $P_H(S_0)$ and $P_H(S_{i+1} \mid s_i)$, for each state $s_i \in S$, and probability distributions $P_H(O_i \mid s_i)$, $s_i \in S$, over the observations, the \emph{HMM$_{off}$} of $H$ is another HMM $Q$ over the same set $S$ of states and same set $O$ of observations.
    The probability distributions of $Q$ are the top-$p$ distributions of the respective ones of $H$, that is,
    \begin{inparaenum}[(i)]
        \item $P_Q(S_0)$ is the top-$p$ distribution of $P_H(S_0)$,
        \item $P_Q(S_{i+1} \mid s_i)$ is the top-$p$ distribution of $P_H(S_{i+1} \mid s_i)$, for each state $s_i \in S$, and
        \item $P_Q(O_i \mid s_i)$ is the top-$p$ distribution of $P_H(O_i \mid s_i)$, $s_i \in S$.
    \end{inparaenum}
\end{definition}

Algorithm~\ref{algorithm:top_p} is pseudo-code for Definition~\ref{definition:top_p_distribution} and returns the top-$p$ distribution for a vector representing a probability distribution.
To get the HMM$_{off}$ for an HMM, Algorithm~\ref{algorithm:top_p} is applied to
\begin{inparaenum}[(i)]
    \item the initial state distribution,
    \item each next state distribution given the current state,
    \item and each observation distribution given the current state.
\end{inparaenum}
We illustrate the HMM$_{off}$ on our running example:

\begin{algorithm}[tb]
\caption{Apply the top-$p$ approach according to Definition~\ref{definition:top_p_distribution} to a vector}\label{algorithm:top_p}
\begin{algorithmic}
\Require $v \in [0,1]^d, p \in (0,1]$ \Comment{\emph{Input} probability distribution and $p$ value}
\Ensure $v' \in [0,1]^d$ is the top-$p$ distribution of $v$
\State $v' \gets \mathbf{0} \in \mathbb{R}^d$
\State $idx \gets \text{sort\_indices\_by\_value\_desc}(v)$ \Comment{create \emph{sorted order} of $v$}
\State $sum \gets 0$
\For{$j$ in $idx$} \Comment{$j$ is an index for $v$ and $v'$}
\State{$sum \gets sum + v[j]$}
\State{$v'[j] \gets v[j]$} \Comment{take events \dots}
\If{$sum \geq p$} \Comment{\dots until accumulated probability reaches $p$}
\State{\textbf{break}}
\EndIf
\EndFor
\State{$v' \gets \frac{v'}{sum}$} \Comment{scale to sum to one}
\end{algorithmic}
\end{algorithm}

\begin{example}
    Table~\ref{table:simple_weather_top_p} shows the state transition probabilities for HMM$_{off}$ and $p=0.7$.
    For the observations, the probability of wearing a raincoat is one for heavy rain and thunderstorm and unaltered for all other states, i.e., $0.4$ for partly cloudy, $0.6$ for light rain, $0.4$ for foggy, and $0.35$ for sunny.
\end{example}

\begingroup
\renewcommand{\arraystretch}{1.1}
\begin{table}[tb]
    \centering
    \begin{tabular}{ccccccc}
        Next/current & \makecell{Partly\\cloudy} & \makecell{Light\\rain} & Foggy & Sunny & \makecell{Heavy\\rain} & Thunderstorm \\ \hline
        Partly cloudy & $\frac{3}{7}$ & 0.25 & $\frac{3}{7}$ & 0.4 & 0 & 0 \\
        Light rain & $\frac{2}{7}$ & 0.25 & $\frac{2}{7}$ & $\frac{1}{3}$ & $\frac{2}{7}$ & $\frac{2}{7}$ \\
        Foggy & 0 & 0 & $\frac{2}{7}$ & 0 & 0 & 0 \\
        Sunny & $\frac{2}{7}$ & 0 & 0 & $\frac{4}{15}$ & 0 & 0 \\
        Heavy rain & 0 & 0.25 & 0 & 0 & $\frac{2}{7}$ & $\frac{3}{7}$ \\
        Thunderstorm & 0 & 0.25 & 0 & 0 & $\frac{3}{7}$ & $\frac{2}{7}$
    \end{tabular}
    \caption{HMM$_{off}$ transition probabilities for the simple weather HMM with $p=0.7$.}
    \label{table:simple_weather_top_p}
\end{table}
\endgroup

Having shifted the probability mass, we exploit the new distributions for efficient inference without enumerating all states.
We explain our prototypical implementation in Section~\ref{section:implementation}.

Definition~\ref{definition:top_p_distribution} does not only give rise to the initially motivated HMM$_{off}$ (c.f. Definition~\ref{definition:top_p_hmm}), but we can also use the top-$p$ distribution to denoise the state distribution at each time step, while applying the exact transition model:
Assume the current state is \emph{partly cloudy} with 0.9 and the remaining probability mass of $0.1$ is distributed among the other states.
Then, by following the top-$p$ idea, we would set the probability of \emph{partly cloudy} to $1$ and the probability for all other states to $0$ and perform the transition to the next time step as defined in the HMM.
Afterward, we again apply the top-$p$ distribution to the obtained new state.
We now formalize this approach:

\begin{definition}[HMM$_{on}$] \label{definition:top_p_state_approach}
    Given an HMM $H$ with a set $S$ of states, a set $O$ of observations, probability distributions over the states given by $P_H(S_0)$ and $P_H(S_{i+1} \mid s_i)$, for each state $s_i \in S$, and probability distributions $P_H(O_i \mid s_i)$, $s_i \in S$, over the observations, the \emph{HMM$_{on}$} of $H$ is another HMM $Q$ over the same set $S$ of states and same set $O$ of observations.
    The observation probabilities are unchanged, i.e., $P_Q(O_i \mid s_i) = P_H(O_i \mid s_i)$, $s_i \in S$.
    However, we define the current state distribution as the top-$p$ distribution of the initial state and as of the outcome of the transition model:
    \begin{align}
        P_Q(S_0) &\coloneq top_p(P_H(S_0)) \\
        P_Q(S_{i+1}) &\coloneq top_p(P_H(S_{i+1} \mid S_i) \cdot P_Q(S_i)).
    \end{align}
\end{definition}

We discuss possible implementations of HMM$_{on}$ in Section~\ref{section:implementation}.
Since we only keep the top-$p$ events, our HMM$_{off}$ and HMM$_{on}$ approaches definitely introduce an error.
For example, in Table~\ref{table:simple_weather_top_p} we see that we introduce an error by dropping the possibilities of transitioning from sunny to heavy rain or thunderstorm.
We next show that the error induced by the two top-$p$ approaches is bounded in terms of $p$

%% file: error.tex
\section{Bounding the Approximation Error}
When we use the HMM$_{off}$ and HMM$_{on}$ approaches, inference is done with altered distributions, dropping events of the original one.
Therefore, there is some error involved.
In this section, we bound the approximation error introduced by the top-$p$ modeling.
The bound can then be used to determine whether to apply the top-$p$ approach or not.
Throughout this section, we use the total variation distance to measure the approximation error:

\begin{definition}[Total Variation]
The \emph{total variation distance} between two discrete probability distributions $P$ and $Q$ over the same set of possible outcomes $X$ is defined as
\begin{equation}
    \delta(P,Q) = \frac{1}{2} \sum_{x \in X} \abs{P(x) - Q(x)}.
\end{equation}
\end{definition}

We choose the total variation distance, because it allows, in contrast to the Kullback-Leibler divergence~\cite{kl_divergence}, arbitrary probabilities being zero.
Moreover, the absolute value can be resolved more easily than, e.g., square roots in the Helliner distance and the total variation fulfills the triangle equality.
We first start by bounding the error of a single time step approximation and continue to bound the error over all time steps based on the result for one time step for both top-$p$ variants.

\subsection{Approximation Error in a Single Time Step}
We show that the error introduced by a top-$p$ distribution is at most $1-p$:

\begin{theorem} \label{theorem:one-top-p}
Given a probability distribution $P$ and the top-$p$ distribution $Q$ of $P$, the total variation distance between $P$ and $Q$ is bounded by
\begin{equation}
    \delta(P,Q) \leq 1-p.
\end{equation}
\end{theorem}
\begin{proof}
    Let us denote with $Y$ the top-$p$ events of $P$.
    We first assume that the probabilities kept exactly match $p$, i.e., $P(Y)=p$ and $Q(Y)=1$.
    Moreover, for each $y \in Y$: $Q(y)=\frac{P(y)}{p}$.
    Now, we can start deriving the distance:
    \begin{align}
        \delta(P,Q) &= \frac{1}{2} \sum_{x \not\in Y} \abs{P(x)} + \frac{1}{2} \sum_{y \in Y} \abs{P(y) - Q(y)} \\
        &= \frac{1}{2} (1-p) + \frac{1}{2} \sum_{y \in Y} \abs{P(y) - \frac{P(y)}{p}} \\
        &= \frac{1}{2} (1-p) + \frac{1}{2} \sum_{y \in Y} \abs{P(y) \left(1- \frac{1}{p}\right)} \\
        &= \frac{1}{2} (1-p) + \frac{1}{2} \sum_{y \in Y} \abs{P(y)} \abs{1-\frac{1}{p}} \\
        &= \frac{1}{2} (1-p) - \left(1 - \frac{1}{p} \right) \cdot \frac{1}{2} \sum_{y \in Y} \abs{P(y)} \\
        &= \frac{1}{2} (1-p) - \left( 1- \frac{1}{p} \right) \cdot \frac{1}{2} p = 1-p.
    \end{align}
    Lifting the assumption of $P(Y)=p$, if the probabilities in the top-$p$ events exceed $p$, i.e., $P(Y) = p' \geq p$, we have $\delta(P,Q) = 1-p' \leq 1-p$.
\end{proof}

Theorem~\ref{theorem:one-top-p} will build the backbone of our proofs.
Let us investigate the additional error introduced by a single transmission through the transition model for both, HMM$_{off}$ and HMM$_{on}$.
We start with HMM$_{on}$:

\begin{theorem}[Multi-Step Bound for HMM$_{on}$] \label{theorem:max_additional_error_top_p_state}
    Given an HMM with a probability distribution $P^k$ over the set of states for each time step $k$ and the corresponding HMM$_{on}$ with a probability distribution $Q^k$ over the set of states for each time step $k$, the total variation distance is
    \begin{equation}
        \delta(P^k,Q^k) \leq (k+1) \cdot (1-p).
    \end{equation}
\end{theorem}
\begin{proof}
The total variation distance is invariant to linear transformations and thus to matrix multiplication when advancing in time.
In each time step, we add at most an error of $1-p$ when applying the top-$p$ distribution.
The claim follows by the triangle inequality.
\end{proof}

A direct consequence of Theorem~\ref{theorem:max_additional_error_top_p_state} is that in each time step, the additional error is at most $1-p$.
We show the same result for HMM$_{off}$, where we do not apply the top-$p$ distribution to the states, but to the transition model.
For the ease of the proofs, we occasionally use the matrix notation of HMMs:
\begin{definition}[Matrix Notation for HMMs] \label{definition:hmm_matrix_notation}
    Given an HMM $H$ with a set $S$ of states, a set $O$ of observations, probability distributions over the states given by $P(S_0)$ and $P(S_{k+1} \mid s_k)$, for each state $s_k \in S$, and probability distributions $P(O_k \mid s_k)$, $s_k \in S$, over the observations, the matrix $T = (P(s_i \mid s_j))_{i,j} , s_i, s_j \in S$ captures the transition model and the current state vector $\hat{s} = P(S_i)$ captures the current probability distribution over the states.
    With $\hat{s}_i$, we denote the $i$-th entry in $\hat{s}$.
    With $T_{i,j}$, we denote $P(s_i \mid s_j)$, and with $T_{:,j}$ the $j$-th column of $T$, i.e., $P(S_{k+1} \mid s_j)$.
    The $L_1$ norm $\norm{\hat{s}}_1$ sums all (absolute) entries, i.e., $\norm{\hat{s}}_1 = \sum_i \abs{\hat{s}_i}$.
\end{definition}

With Definition~\ref{definition:hmm_matrix_notation}, we get
\begin{equation}
    P(S_i) \cdot P(S_{i+1} \mid S_i) = P(S_{i+1}) = \hat{s}_{i+1} = T \cdot \hat{s}_i.
\end{equation}
Now, let us take a look at the error incurred when advancing one time step:

\begin{theorem}[One-Step Bound for HMM$_{off}$] \label{theorem:one_step_top_p}
    Given an HMM $H$ with probability distribution $P_H$ and matrix notation $T_H$ and the HMM$_{off}$ $Q$ with probability distribution $P_Q$ and matrix notation $T_Q$ over the same set of states $S$, and given a state distribution $\hat{s}$, the total variation distance between the results of advancing $\hat{s}$ in $H$ and in $Q$ is
    \begin{equation}
        \delta(T_H \cdot \hat{s}, T_Q \cdot \hat{s}) \leq 1-p.
    \end{equation}
\end{theorem}
\begin{proof}
    \begin{align}
        \delta(T_H \cdot \hat{s}, T_Q \cdot \hat{s}) &= \frac{1}{2} \norm{T_H \cdot \hat{s} - T_Q \cdot \hat{s}}_1 \\
        &= \frac{1}{2} \norm{ \sum_j{T_{H_{:,j}}} \hat{s}_j - \sum_j{T_{Q_{:,j}} \hat{s}_j}}_1 \\
        &= \frac{1}{2} \norm{\sum_j{\left(T_{H_{:,j}} - T_{Q_{:,j}} \right) \cdot \hat{s}_j}}_1 \\
        &\leq \sum_j{ \frac{1}{2} \norm{T_{H_{:,j}} - T_{Q_{:,j}}}_1 \cdot \hat{s}_j} \\
        &= \sum_j{ \delta(P_H(S \mid s_j), P_Q(S \mid s_j)) \cdot \hat{s}_j} \\
        &= \sum_j (1-p) \cdot \hat{s}_j = 1-p
    \end{align}
\end{proof}

Theorem~\ref{theorem:one_step_top_p} yields the same upper bound for the HMM$_{off}$ approach as Theorem~\ref{theorem:max_additional_error_top_p_state} does for the HMM$_{on}$ approach, i.e., $(k+1) \cdot (1-p)$, again by using the triangle inequality.
Unfortunately, the bound of $(k+1) \cdot (1-p)$ quickly approaches the trivial bound of $1$.
Boyen and Koller show that the error in terms of Kullback-Leibler divergence is bounded~\cite{error_bound}.
However, the Kullback-Leibler divergence is not applicable to our top-$p$ approaches, because we set probabilities to zero and the Kullback-Leibler divergence divides by probabilities.
Nevertheless, we use the key idea of Boyen and Koller~\cite{error_bound} to show the same bound for our total variance distance.

\subsection{Overall Approximation Error}
Boyen and Koller show for the Kullback-Leibler divergence that the error in each time step is at most $\frac{\epsilon}{\gamma}$, where $\gamma$ is a parameter inherent to the specific HMM and $\epsilon$ is the error incurred in each time step~\cite{error_bound}.
The parameter $\gamma$ of an HMM tells the minimum probability mass on which two different state distributions agree on after a single time step in the same model.
While the bounded error result is promising, we cannot use their theorems directly, because of their use of the Kullback-Leibler divergence, which is not applicable in our case.
The Kullback-Leibler divergence divides by probabilities and we set probabilities to zero, rendering the Kullback-Leibler divergence not applicable in our case.
Therefore, we use the key idea from Boyen and Koller~\cite{error_bound} to bound the error of our HMM$_{on}$ and HMM$_{off}$ approaches in terms of our used total variation distance.
We start by defining the parameter $\gamma$ and continue by proving that an error incurred in one time step diminishes by a factor of $1-\gamma$ over the next time steps.
Thereby, the crucial property is that the diverging states, i.e., the \emph{true} distribution and the \emph{approximated} distribution, are fetched through the \emph{same} transition model as in HMM$_{on}$.
Given the HMM$_{off}$ approach, the two state distributions are fetched through \emph{different} transition models.
However, in our proof, we later show that we can still use the minimal mixing rate and arrive at the diminishing factor.
Then, given that our error increases at most by $1-p$ in each time step (c.f. Theorems~\ref{theorem:max_additional_error_top_p_state} and~\ref{theorem:one_step_top_p}), we build a geometric series and prove that our overall approximation error in terms of total variation distance is at most $\frac{1-p}{\gamma}$ in each time step.

Let us first prove that an error introduced in one time step diminishes by a constant factor per time step.
For this, we define the \emph{minimal mixing rate} $\gamma$, which directly links to the diminishing factor of $1-\gamma$.
The minimal mixing rate sets the minimum probability mass two different state distributions agree on after a single time step~\cite{error_bound}.

\begin{definition}[Definition 3 from Boyen and Koller~\cite{error_bound}]
    For a Markov process with stochastic transition model $Q$ with states $\omega_i$, the \emph{minimal mixing rate} of $Q$ is
    \begin{equation}
        \gamma \coloneq \min_{i_1,i_2} \sum_j \min\left\{Q\left(\omega_j \mid \omega_{i_1}\right), Q\left(\omega_j \mid \omega_{i_2}\right)\right\}.
    \end{equation}
\end{definition}

For showing the diminishing factor of $1-\gamma$ for the error, we split the transition from one time step to the next into two transitions containing an intermediate state.
We are primarily interested in the first transition, as we show later, using the following lemma, that this one diminishes the error.

\begin{lemma}[Lemma 2 and Theorem 3 by Boyen and Koller~\cite{error_bound}] \label{lemma:intermediate_state}
    Fixing two state distributions $\varphi$ and $\psi$, the transition of a stochastic process $Q$ can be split into two steps:
    First, a transition defined by $R^{\Gamma}$ from $\Omega$ to $\tilde{\Omega}$ and, second, a transition defined by $R^{\Delta}$ from $\tilde{\Omega}$ to $\Omega$, where $\Omega = \left\{\omega_i\right\}_i$ is the state space of $Q$ and $\tilde{\Omega} = \Omega \cup \{ c \}$ with a new intermediate state $c$.
    The process $R^{\Gamma}$ preserves its state with probability $1-\gamma$ and transitions to state $c$ with probability $\gamma$.
    The process $R^{\Delta}$ transitions from $\omega_i$ to $\omega_j$ with probability $\frac{Q_{i,j}^{\Delta}}{1-\gamma}$.
    For state $c$, $R^{\Delta}$ transitions to state $\omega_j$ with probability $\sum_i \varphi(\omega_i) \frac{Q_{i,j}^{\Gamma}}{\gamma}$.
    The matrices $Q^{\Gamma}$ and $Q^{\Delta}$ arise from an additive contraction decomposition $Q = Q^{\Gamma} + Q^{\Delta}$, where $Q^{\Gamma}$ and $Q^{\Delta}$ are non-negative matrices such that, for all $i$, $\sum_j Q_{i,j}^{\Gamma} = \gamma$, and for all $j$, $\sum_i \varphi(\omega_i) Q_{i,j}^{\Gamma}=\sum_i \psi(\omega_i) Q_{i,j}^{\Gamma}$.
\end{lemma}

We now show that the first transition diminishes the error by $1-\gamma$ and that the second one does not increase the error:
\begin{theorem}[Diminishing Error in Stochastic Process] \label{theorem:diminishing-error}
    For two state distributions $\varphi$ and $\psi$ and their counterparts $\varphi'$ and $\psi'$ in the next time step as defined by a stochastic process $Q$:
    \begin{equation}
        \delta(\varphi',\psi') \leq \left(1-\gamma\right) \delta(\varphi, \psi).
    \end{equation}
\end{theorem}
\begin{proof}
    We fix $\varphi$ and $\psi$.
    Using Lemma~\ref{lemma:intermediate_state}, let $\tilde{\varphi}$ and $\tilde{\psi}$ denote the respective distributions in the intermediate step.
    We show that our claim holds by showing the error $\delta(\tilde{\varphi},\tilde{\psi}) = (1-\gamma) \delta(\varphi,\psi)$ in the intermediate step and the error $\delta(\varphi',\psi') \leq \delta(\tilde{\varphi},\tilde{\psi})$ in the next time step.
    The last inequality holds because the total variation is invariant to linear transformations.

    For showing $\delta(\tilde{\varphi},\tilde{\psi}) = (1-\gamma) \delta(\varphi,\psi)$, we first note
    \begin{align}
        \tilde{\varphi}(c) &= \sum_i \gamma \varphi(\omega_i) = \gamma \quad\text{and} \\
        \tilde{\varphi}(\omega_i) &= (1-\gamma) \varphi(\omega_i),
    \end{align}
    which also hold for $\tilde{\psi}$.
    Then, we have
    \begin{align}
        \delta(\tilde{\varphi},\tilde{\psi}) &= \frac{1}{2} \sum_i \abs{\tilde{\varphi}(\omega_i) - \tilde{\psi}(\omega_i)} + \frac{1}{2} \abs{\tilde{\varphi}(c) - \tilde{\psi}(c)} \\
        &= \frac{1}{2} \sum_i \abs{\left(1-\gamma\right) \varphi(\omega_i) - \left(1-\gamma\right) \psi(\omega_i)} \\
        &= \left(1-\gamma\right) \delta(\varphi,\psi).
    \end{align}
\end{proof}

While we now have already everything at hand for proving the error bound for HMM$_{on}$, we still lack the applicability of $\gamma$ to HMM$_{off}$, since the two state distributions are not transitioned via the \emph{same} stochastic process.
We now show the diminishing factor for the top-$p$ approach, by first using the triangle inequality and second using Theorems~\ref{theorem:diminishing-error} and~\ref{theorem:one_step_top_p}:

\begin{theorem}[Diminishing Error in HMM$_{off}$] \label{theorem:diminishing-error-top_p}
    Given an HMM $H$ with matrix notation $T_H$ and the corresponding HMM$_{off}$ $Q$ with matrix notation $T_Q$, and given two state distributions $\hat{s}$ and $\hat{u}$, the total variation distance after a single transition is bounded by
    \begin{align}
        \delta(T_H \cdot \hat{s}, T_Q \cdot \hat{u}) &\leq \delta(T_H \cdot \hat{s}, T_H \cdot \hat{u}) + \delta(T_H \cdot \hat{u}, T_Q \cdot \hat{u}) \\
        &\leq (1 - \gamma) \cdot \delta(\hat{s}, \hat{u}) + \delta(T_H \cdot \hat{u}, T_Q \cdot \hat{u}) \\
        &\leq (1 - \gamma) \cdot \delta(\hat{s}, \hat{u}) + (1 - p).
    \end{align}
\end{theorem}

At this point, we have already shown that we
\begin{inparaenum}[(i)]
\item introduce an approximation error of $1-p$ when using the top-$p$ distribution once (c.f. Theorem~\ref{theorem:one-top-p}),
\item the error introduced in one time step diminishes by a factor of $1-\gamma$ for each future time step (c.f. Theorem~\ref{theorem:diminishing-error} for HMM$_{on}$ and Theorem~\ref{theorem:diminishing-error-top_p} for HMM$_{off}$), which is $(1-p)(1-\gamma)^k$ for time step $k$, and
\item in each time step, we only add an additional error of at most $1-p$ (c.f. Theorem~\ref{theorem:max_additional_error_top_p_state} for HMM$_{on}$ and Theorem~\ref{theorem:max_additional_error_top_p_state} for HMM$_{off}$).
\end{inparaenum}
We combine these results in our overall approximation guarantee for the top-$p$ approaches by analyzing the induced geometric series:

\begin{theorem}[Error Bound for HMM$_{off}$ and HMM$_{on}$] \label{theorem:error_bound}
    The total variation distance between the state distributions $Q^k$ and $P^k$ for the HMM$_{off}$ or HMM$_{on}$ $Q$ and the original HMM $P$ in each time step $k$, respectively, is
    \begin{equation}
        \delta(P^k, Q^k) \leq \frac{1-p}{\gamma},
    \end{equation}
    where $\gamma$ is the minimal mixing rate of $P$.
\end{theorem}
\begin{proof}
We prove the claim by induction: In the first time step, the error is $1-p \leq \frac{1-p}{\gamma}$ by Theorem~\ref{theorem:one-top-p} and because of $\gamma \leq 1$.
\paragraph{Induction Step}
In time step $k$, because of the triangle inequality, the error consists of the error from time step $k-1$ and an additional error of at most $1-p$.
However, the error of time step $k-1$ is, by Theorems~\ref{theorem:diminishing-error} and~\ref{theorem:diminishing-error-top_p}, diminished by the factor $1-\gamma$.
Therefore, the error in time step $k$ is
\begin{align}
    (1-\gamma) \cdot \frac{1-p}{\gamma} + 1-p &= (1-\gamma) \cdot \frac{1-p}{\gamma} + \gamma \cdot \frac{1-p}{\gamma} \\
    &= (1-\gamma + \gamma) \cdot \frac{1-p}{\gamma} = \frac{1-p}{\gamma}.
\end{align}
\end{proof}

Please note that Theorem~\ref{theorem:error_bound} holds for both, HMM$_{off}$ and HMM$_{on}$, approaches.
We highlight that Theorem~\ref{theorem:error_bound} helps an agent to decide whether to use the top-$p$ approach beforehand based on the error bound.
Moreover, the exact total variation can be calculated for individual HMMs.
In the next section, we substantiate our theoretical margin of error empirically, show a reduced runtime for inference, and an increased sparsity.

%% file: evaluation.tex
\section{Evaluation} \label{section:evaluation}
We evaluate the top-$p$ approach on three HMMs: One favoring our approach, one that is disadvantageous for the top-$p$ approach, and a language model for a realistic application.
We run all tests on a 13th Gen Intel(R) Core(TM) i5-1345U CPU with 1.60 GHz and 16 GB of RAM.
We use NumPy~\cite{numpy} and SciPy~\cite{scipy} for matrix and sparse matrix operations, respectively.

For the different HMMs, we have trained a very simple language model (\emph{LM HMM}) as an HMM with 7620 states.
In practice, one would distill a larger model into an HMM~\cite{gelato}.
However, in this evaluation, we focus on the runtime and not on the quality of the language model.
We create a synthetic \emph{Bell HMM}, which is designed to favor the top-$p$ approach.
The Bell HMM distributes, in each next-state distribution, a probability mass of 0.9 over five states and the remaining 0.1 over 795 states, favoring our approach by having lots of unlikely successor states.
For a synthetic HMM disadvantageous for the top-$p$ approach, we create a \emph{Uniform HMM} consisting of uniform state distributions for a system of 800 states.
The number of observations equals the number of states in all HMMs.

We split the evaluation in three parts:
First, we discuss our implementation of HMM$_{off}$ and HMM$_{on}$.
Second, we compare the different top-$p$ approaches, namely
\begin{inparaenum}[(i)]
\item HMM$_{off}$ from Definition~\ref{definition:top_p_hmm},
\item HMM$_{on}$ from Definition~\ref{definition:top_p_state_approach}, and
\item HMM$_{off+on}$, being the combination of the two.
\end{inparaenum}
Afterward, we delve into an in-depth evaluation of HMM$_{off}$.

\subsection{Implementation} \label{section:implementation}
Both approaches, HMM$_{off}$ and HMM$_{on}$, introduce sparseness among the transition model or the state distribution, respectively.
Traditional inference in HMMs involves matrix-vector multiplication, which is heavily optimized in, e.g., NumPy.
We discuss our prototypical implementations for the two approaches in this subsection, before using them in the evaluation.

We start with HMM$_{off}$.
Since the transition matrix is more sparse than the original transition matrix, we use sparse matrices to only store and enumerate the non-zero entries.
We explain how to perform inference, by (sparse) matrix multiplication, in an HMM$_{off}$ $Q$ of an original HMM:
\begin{inparaenum}[(1)]
\item The HMM$_{off}$ $Q$ is built according to Definition~\ref{definition:top_p_hmm} by using Algorithm~\ref{algorithm:top_p}.
\item The obtained transition probabilities and observation probabilities are stored in a transition and observation matrix, respectively.
\item The two matrices are each transformed into \emph{compressed row storage}~\cite{barrett1994templates}:
For each row in the matrix, a tuple $(c,x)$ for each non-zero entry, where $c$ is the column number and $x$ the value.
\item For inference, sparse matrix-vector multiplication~\cite{blelloch1996programming} is used to proceed the current state distribution in time and to get the observation distribution:
The dot product between each row and the vector is calculated by summing the products of $x$ and the entry $c$ of the vector for each $(c,x)$ tuple in the row.
Sparse matrix-vector multiplication runs in time proportional to the number of non-zero entries~\cite{blelloch1996programming} and is also available for GPUs~\cite{bell2008efficient}.
\end{inparaenum}
We use SciPy~\cite{scipy} for sparse matrices and operations on them.
Regular matrix-vector multiplication runs in time $\mathcal{O}(n^2)$, so sparse matrix multiplication pays off when the sparsity is high.

We now switch to HMM$_{on}$.
Here, using sparse vectors or matrices is a lot trickier, since the state vector is, when concerning the pure number of values, smaller than the transition matrix.
When applying the HMM$_{on}$ approach, we sort the state distribution in each time step for extracting the top-$p$ distribution, leading to an additional runtime of $\mathcal{O}(n \log n)$.
The hypothesis is that the processing pays off for faster inference.
We see three (naive) possible implementations:
First, basic matrix-vector multiplication without any sparse multiplications: Apply the top-$p$ distribution to the state distribution in each time step and keep using standard NumPy~\cite{numpy} matrix-vector multiplication.
The second alternative is to use masks for matrices and arrays to skip enumerating them during multiplication.
However, masking and saving the result still requires some time iterating the vector.
The third alternative is to precompute all possible matrices arising when leaving out any subset of columns and reusing the precomputed matrices for sparse multiplication.
Obviously, the precomputation introduces an exponential overhead.
But even the overhead does not pay off, as the surrounding operations still make the online inference slower than traditional matrix-vector multiplication.

In our preliminary tests, standard matrix-vector multiplication (first alternative) is still the fastest among the two other possible implementations.
In theory, at least the third approach should be faster when measuring only the online runtime.
In practice, the overhead for sorting is too big compared to the heavily optimized matrix multiplication implementation in NumPy.
Therefore, excessive work is required to implement the HMM$_{on}$ approach efficiently.
Since we see ourselves unable to compete with NumPy in the short run, we use the naive implementation.
Next, we use our implementations to compare HMM$_{off}$ and HMM$_{on}$.

\subsection{Comparison of HMM$_{off}$ and HMM$_{on}$ Approaches}

In this subsection, we compare the different top-$p$ approaches.
Let us briefly define the combination of HMM$_{off}$ and HMM$_{on}$:
For the \emph{HMM$_{off+on}$} approach, given an HMM with $P(S_t \mid S_{t-1})$, $S_0$, and the respective HMM$_{off}$ $P_Q(S_t \mid S_{t-1})$, the probabilities $P''$ of the state variables are given by the recursion
\begin{align}
    P''(S_t) &\coloneq top_p\left(P_Q(S_t \mid S_{t-1}) \cdot P''(S_{t-1}\right),\\
    P''(S_0) &\coloneq top_p(P(S_0)) = P_Q(S_0).
\end{align}

Throughout this subsection we fix $p=0.9$.
We start by analyzing the runtime and continue with the error in terms of total variation distance.

\paragraph{Runtime for Inference}
For the Bell HMM, the proposed HMM$_{off}$ approach is the fastest, followed by HMM$_{off+on}$.
Contrary, the HMM$_{on}$ approach is even slower than inference in the original HMM.
In the Uniform HMM, all approaches are slower than inference in the original HMM.
However, this time, the HMM$_{on}$ approach is the fastest among the variants, followed by HMM$_{off}$ and HMM$_{off+on}$ being the slowest.
For the LM HMM, the results are comparable with the Bell HMM.
Figure~\ref{figure:runtime_exact_variants} shows the results.

\begin{figure}[tbp]
     \centering
     \begin{subfigure}[b]{0.6\textwidth}
         \centering
         \includegraphics[width=\textwidth]{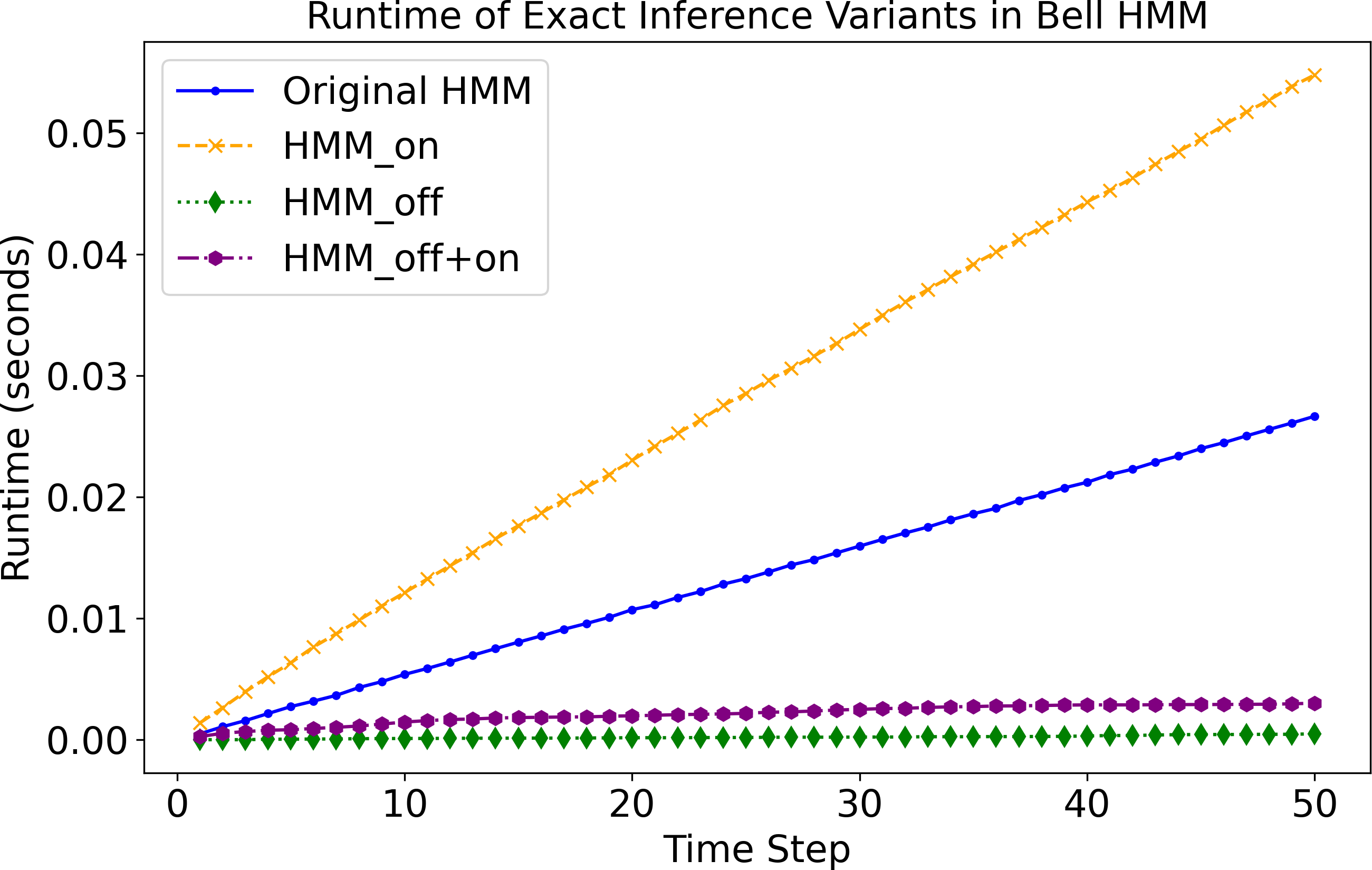}
         \caption{Bell HMM}
         \label{figure:bell_runtime_exact_variants}
     \end{subfigure}
     \hfill
     \begin{subfigure}[b]{0.6\textwidth}
         \centering
         \includegraphics[width=\textwidth]{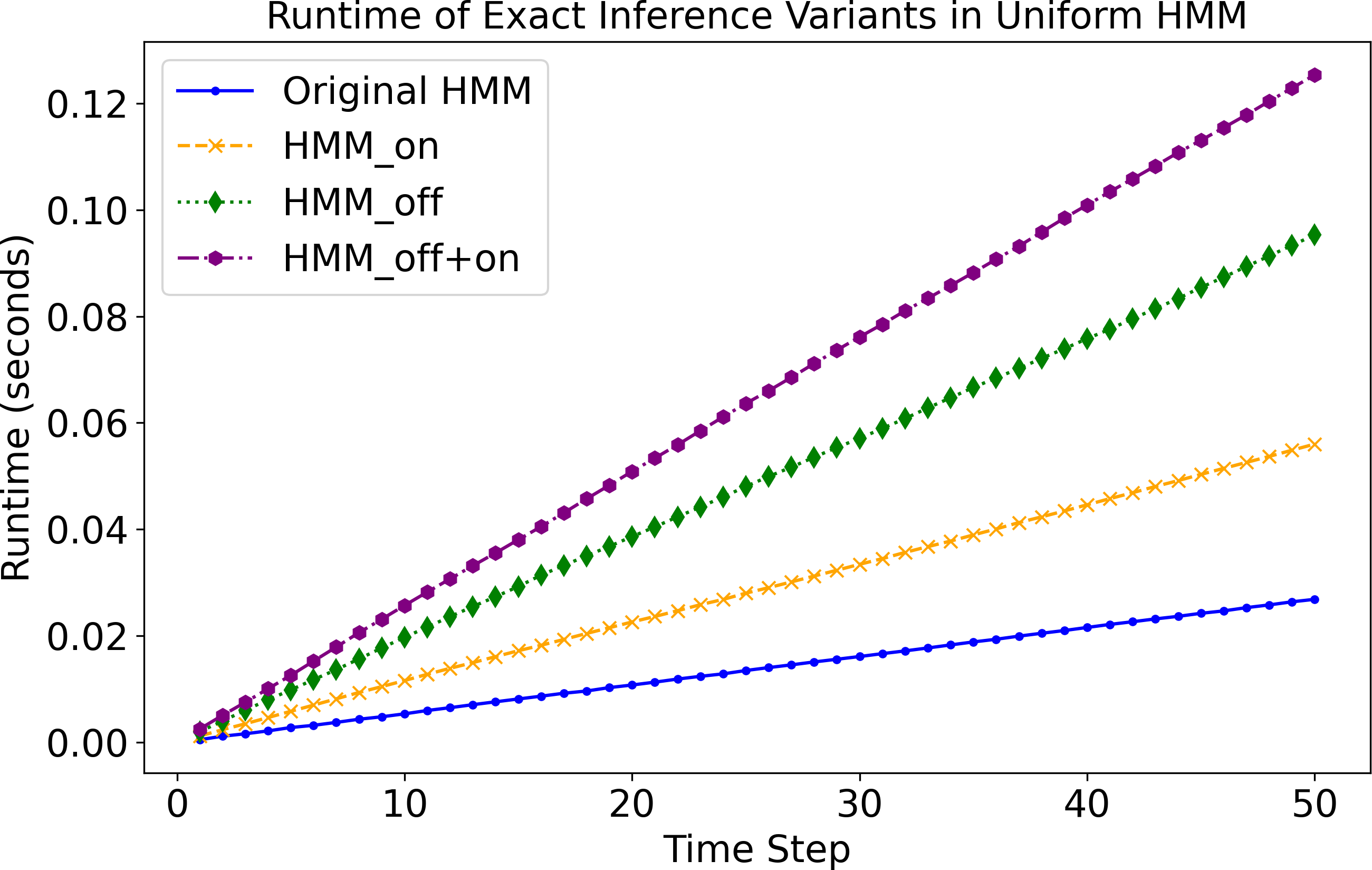}
         \caption{Uniform HMM}
         \label{figure:uniform_runtime_exact_variants}
     \end{subfigure}
     \hfill
     \begin{subfigure}[b]{0.6\textwidth}
         \centering
         \includegraphics[width=\textwidth]{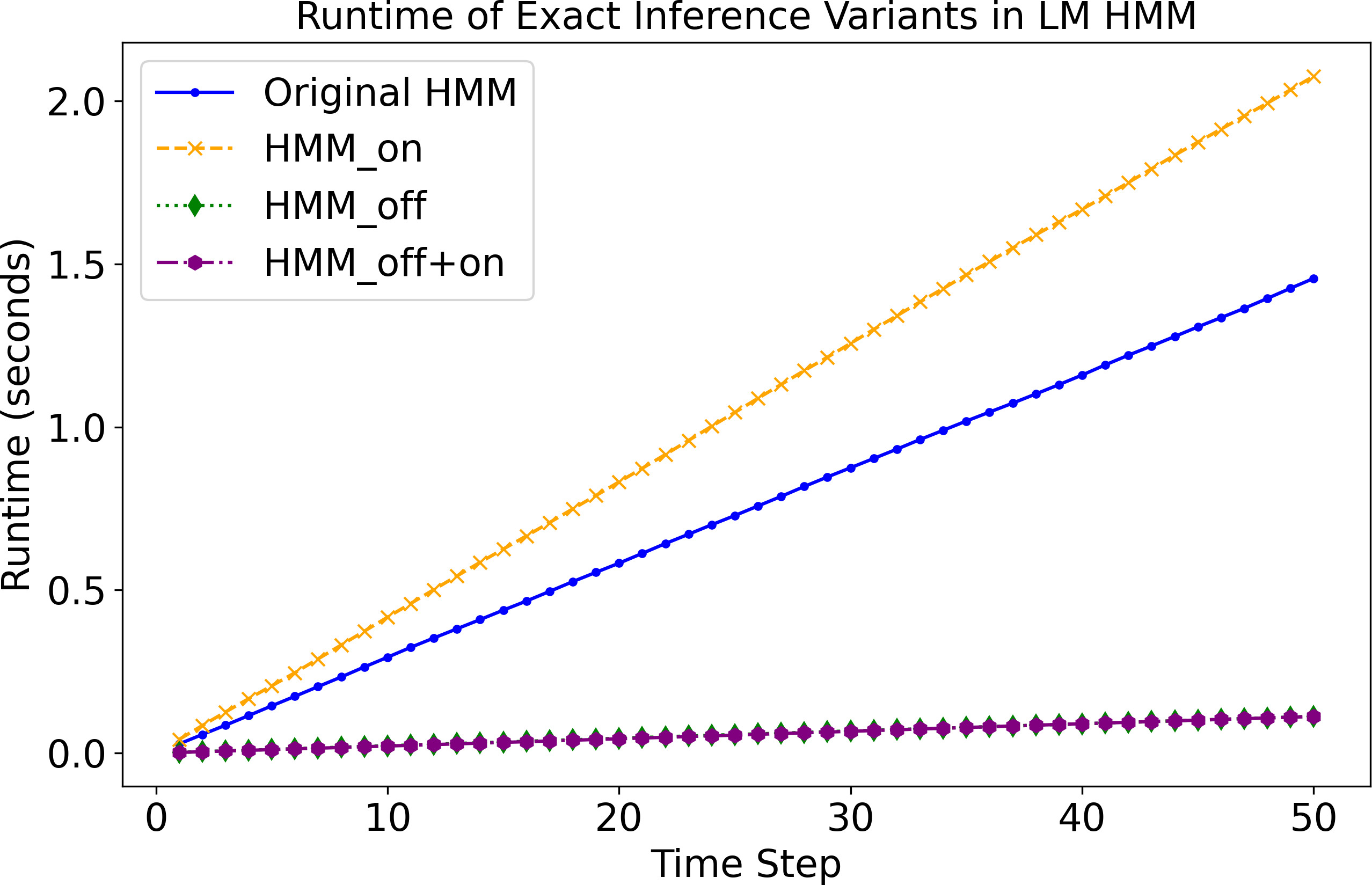}
         \caption{LM HMM}
         \label{figure:lm_runtime_exact_variants}
     \end{subfigure}
     \hfill
     \caption{Runtime for inference for the different top-$p$ variants in the test HMMs. The blue line shows the runtime in the original HMM, the green for our proposed HMM$_{off}$, yellow for HMM$_{on}$, and purple for HMM$_{off+on}$.}
     \label{figure:runtime_exact_variants}
\end{figure}

\paragraph{Runtime for Filtering}
We enter observations every five time steps.
We omit the LM HMM, since observation in text generation is not useful.
In both HMMs, Bell and Uniform, the HMM$_{off}$ approach is the only one being faster than in the original HMM.
Figure~\ref{figure:runtime_filtering_variants} shows the results.

\begin{figure}[tbp]
     \centering
     \begin{subfigure}[b]{0.8\textwidth}
         \centering
         \includegraphics[width=\textwidth]{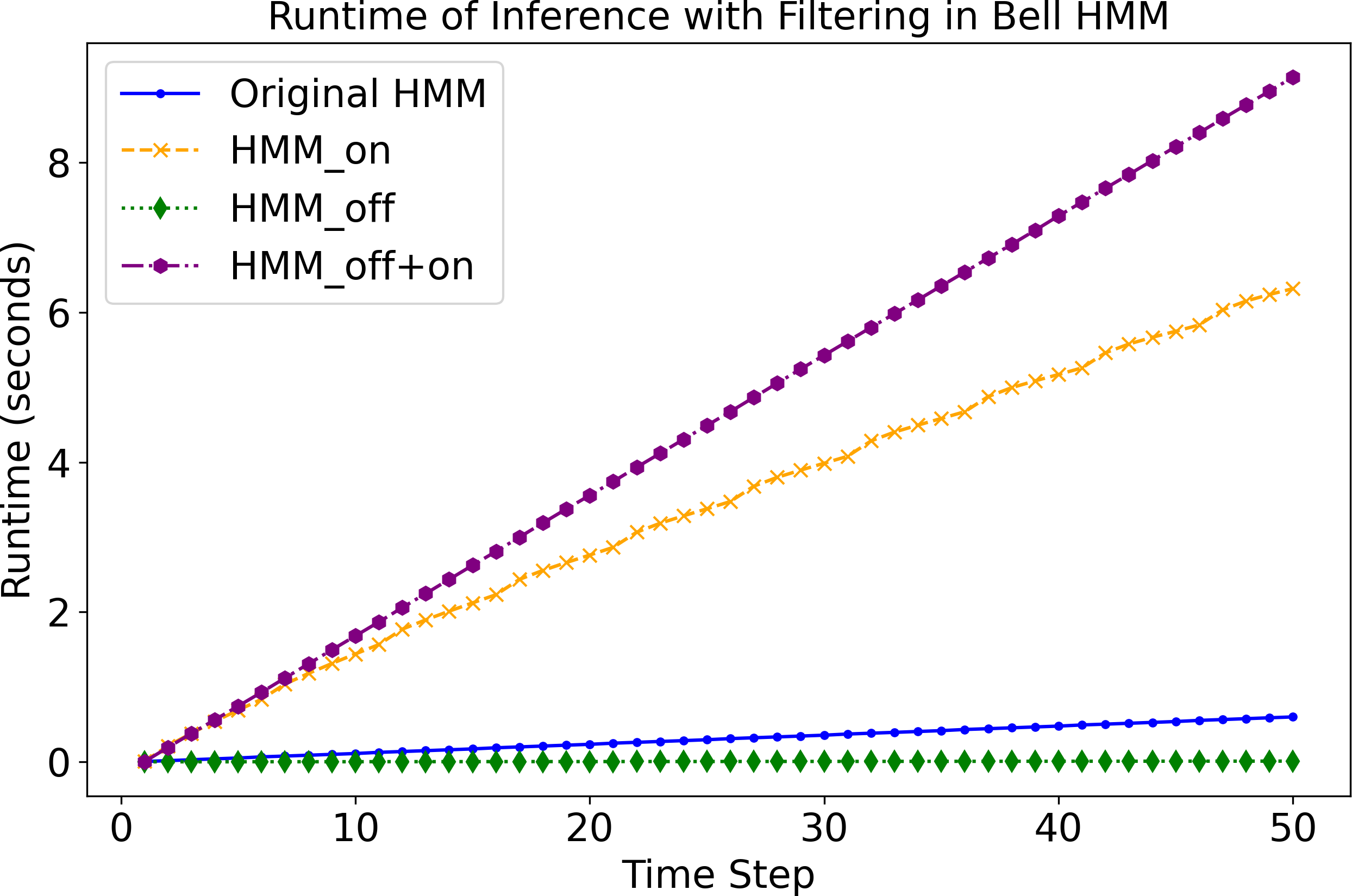}
         \caption{Bell HMM with filtering}
         \label{figure:bell_runtime_filtering_variants}
     \end{subfigure}
     \hfill
     \begin{subfigure}[b]{0.8\textwidth}
         \centering
         \includegraphics[width=\textwidth]{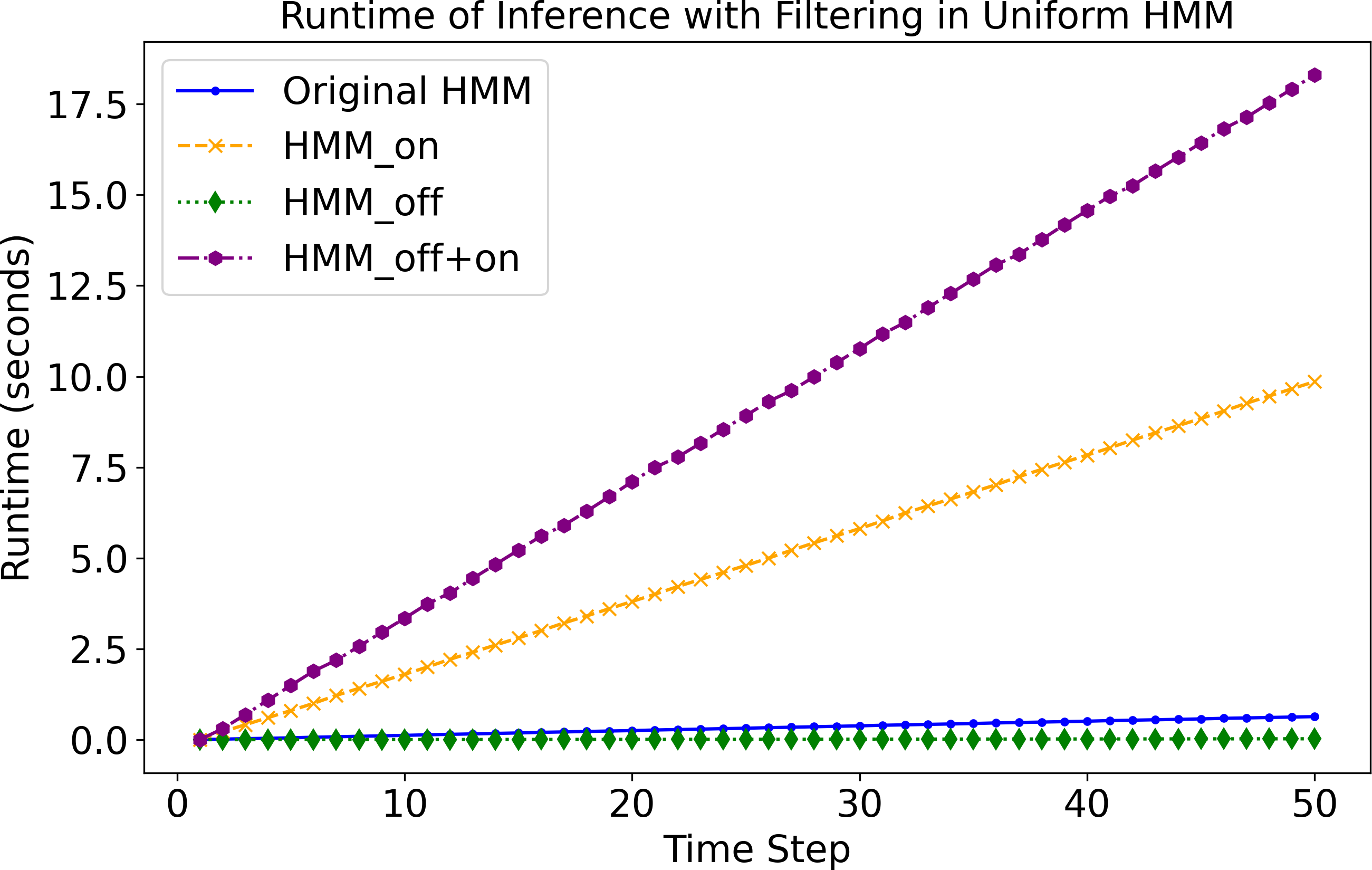}
         \caption{Uniform HMM with filtering}
         \label{figure:uniform_runtime_filtering_variants}
     \end{subfigure}
     \caption{Runtime for filtering in two test HMMs. The blue line shows the runtime in the original HMM, the green for our proposed HMM$_{off}$ approach, yellow for HMM$_{on}$, and purple for HMM$_{off+on}$.}
     \label{figure:runtime_filtering_variants}
\end{figure}

\paragraph{Total Variation Distance}
The results are somewhat similar for all three HMMs:
The HMM$_{off+on}$ approach has the biggest error, followed by the HMM$_{on}$ approach.
The HMM$_{off}$ approach has the smallest error in all three HMMs.
In the Uniform HMM, HMM$_{off}$ and HMM$_{on}$ have the same error.
In the LM HMM, however, the error of the HMM$_{on}$ is more than $1.9$ times larger than for HMM$_{off}$.
And in the Bell HMM, the error of the HMM$_{on}$ is even more than $9$ times larger than for HMM$_{off}$.
Figure~\ref{figure:total_variation_variants} shows the results.
Sine the HMM$_{on}$ approach seems to lead to a higher error, we perform small test on the running weather example (c.f. Table~\ref{table:simple_weather}):
After approximately 25 time steps, HMM$_{off}$ and HMM$_{on}$ converge to the same error, while HMM$_{on}$ converging slower.

\begin{figure}[tbp]
     \centering
     \begin{subfigure}[b]{0.6\textwidth}
         \centering
         \includegraphics[width=\textwidth]{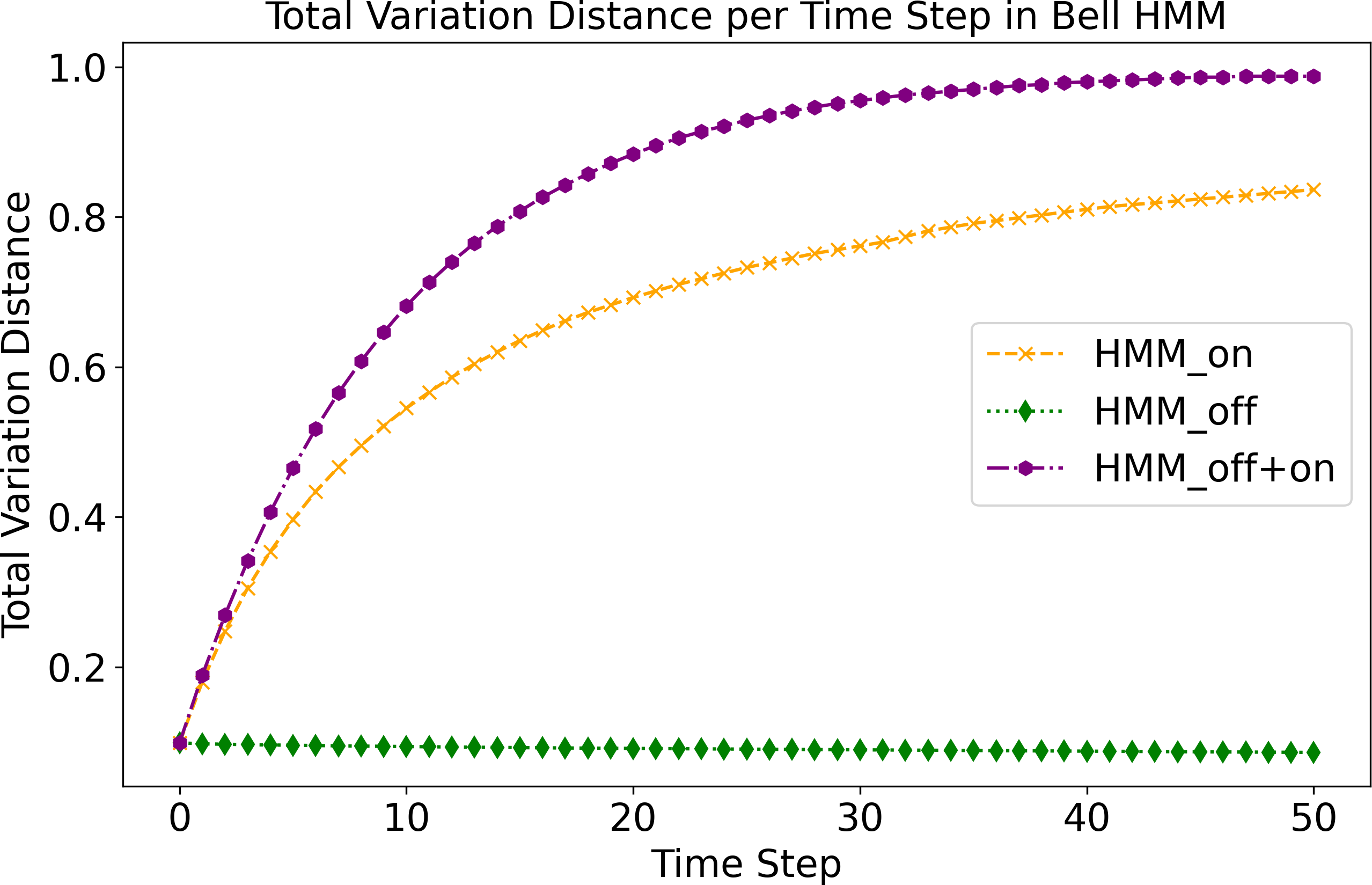}
         \caption{Bell HMM}
         \label{figure:bell_total_variation_variants}
     \end{subfigure}
     \hfill
     \begin{subfigure}[b]{0.6\textwidth}
         \centering
         \includegraphics[width=\textwidth]{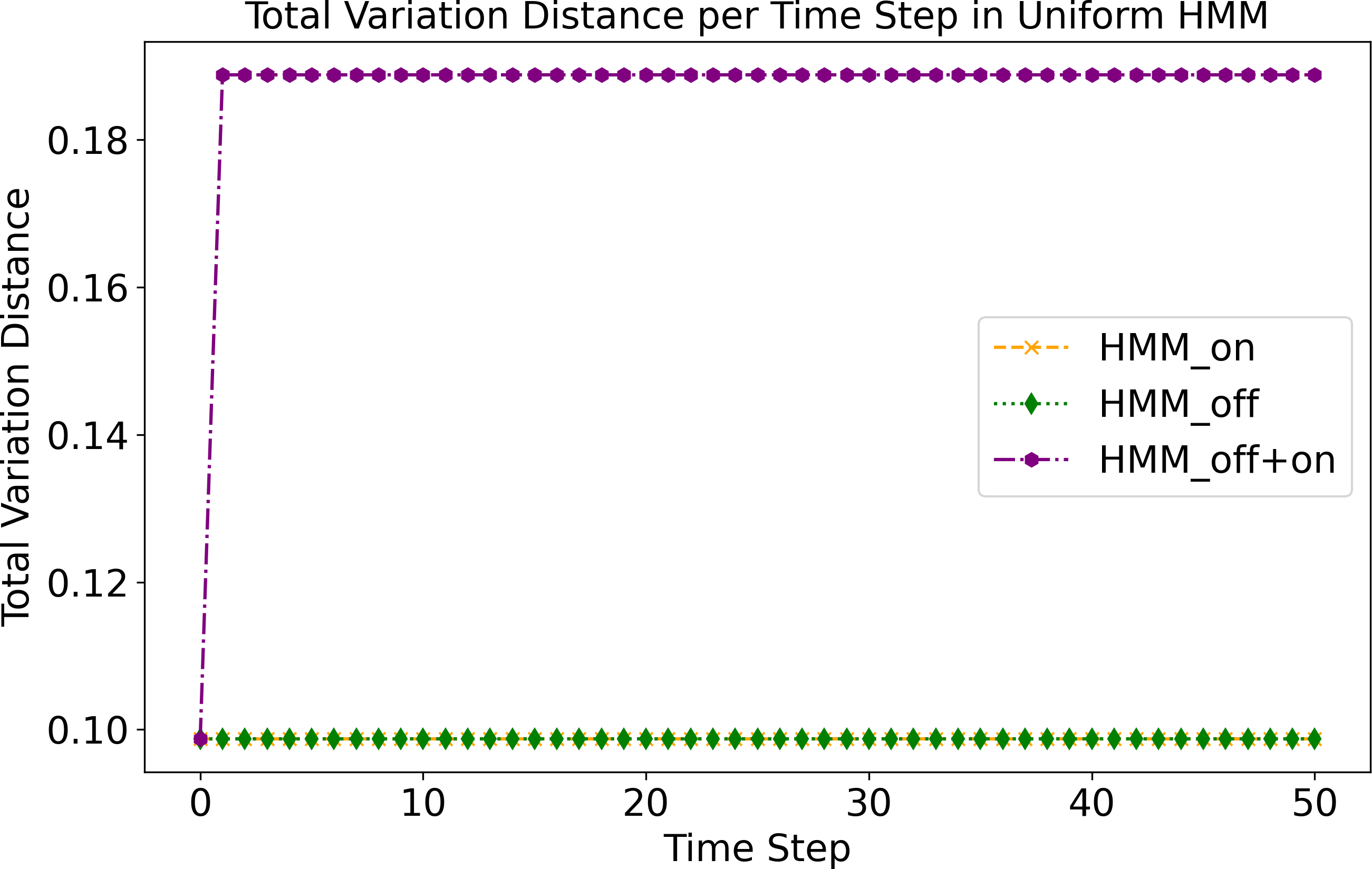}
         \caption{Uniform HMM}
         \label{figure:uniform_total_variation_variants}
     \end{subfigure}
     \hfill
     \begin{subfigure}[b]{0.6\textwidth}
         \centering
         \includegraphics[width=\textwidth]{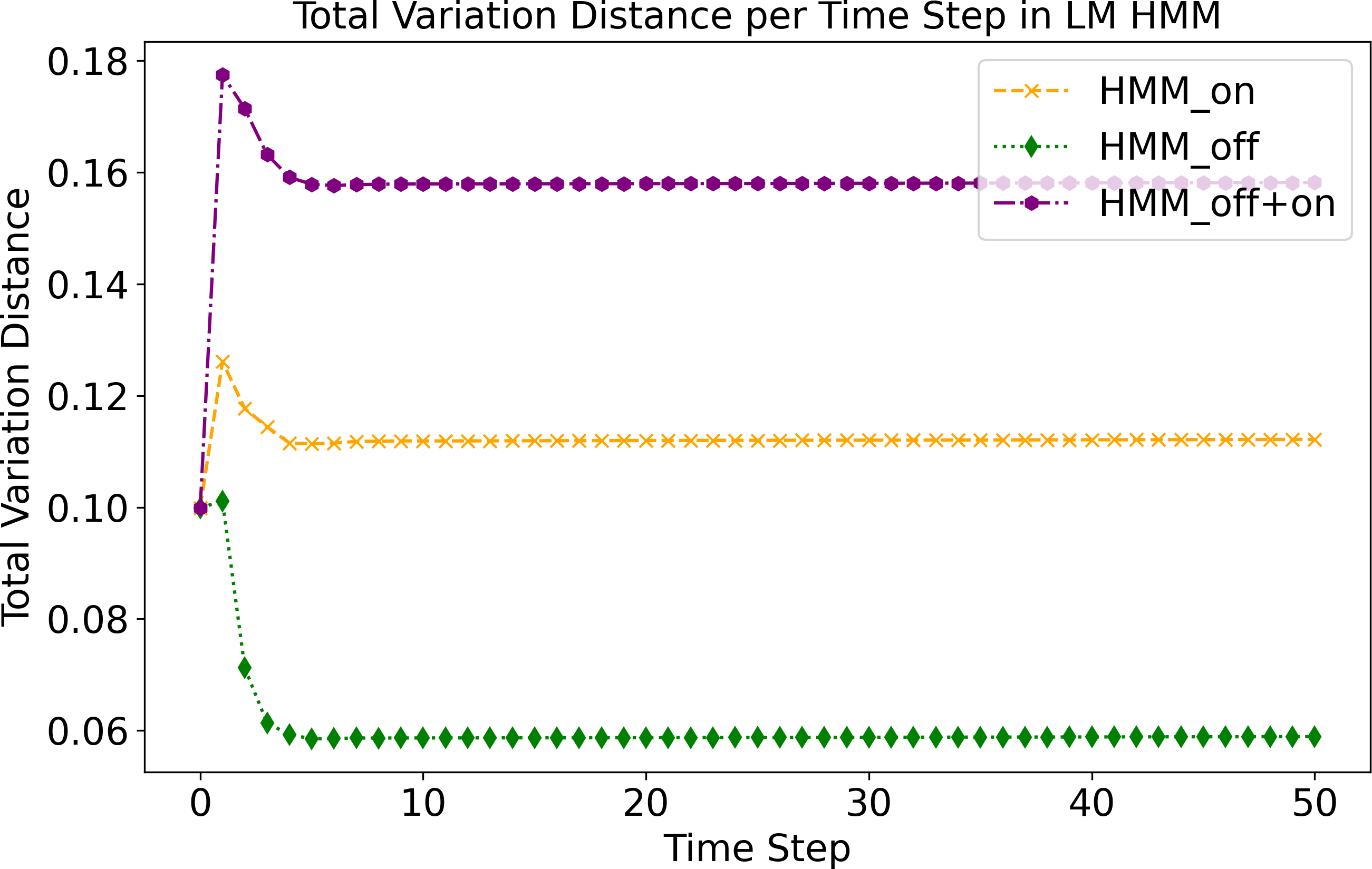}
         \caption{LM HMM}
         \label{figure:lm_total_variation_variants}
     \end{subfigure}
     \hfill
     \caption{Total variation distance for the different top-$p$ variants in the test HMMs. The green line shows the total variation distance for our proposed HMM$_{off}$ approach, yellow for HMM$_{on}$, and purple for HMM$_{off+on}$.}
     \label{figure:total_variation_variants}
\end{figure}

When analyzing the evaluation of the variants, we make two observations:
First, applying the top-$p$ distribution in each time step, i.e., for HMM$_{on}$ and HMM$_{off+on}$, is time-consuming and thus slower than the respective variant without enlarged computation in each time step.
Second, the HMM$_{on}$ approach differs in the empirical error while having the same theoretical bound.
In the Bell HMM, each state is constituted by itself and the neighboring states.
With the HMM$_{off}$ approach, we conserve this constitution.
But if some of the neighboring states are set to $0$ due to the HMM$_{on}$ approach, the state gets farther away from the original model.
In other models, such as the simple weather example, such a structure is not existent, and the error is smaller than with HMM$_{off}$.
However, all in all, HMM$_{off}$ appears to be the more promising as it is the fastest and its error is bounded.
Therefore, we evaluate HMM$_{off}$ in-depth in the next subsection.

\subsection{Evaluation of the HMM$_{off}$ Approach}
We investigate the sparsity, the runtime with and without observations, and the total variation distance of HMM$_{off}$.
We evaluate HMM$_{off}$ for the $p$ values $0.5$, $0.7$, and $0.9$ and for the maximum time step $50$.
We call the corresponding approaches top-$0.5$, top-$0.7$, and top-$0.9$ throughout this subsection.
In the runtime evaluation, we solely report numbers for $p=0,9$, as the three top-$p$ approaches are mostly comparable.
For the error investigation, we also report numbers for the other tested $p$ values.

\paragraph{Sparsity}
We define the sparsity as the ratio of zero-entries in the transition matrix.
In the Bell HMM, all three top-$p$ approaches lead to a sparsity of above $0.99$.
In the Uniform HMM, the sparsity is approximately $1-p$.
And in the LM HMM, the sparsity is for all three top-$p$ greater than $0.97$.

\paragraph{Runtime for Inference}
The top-$0.9$ HMM for the Bell HMM outperforms exact inference in the original HMM by a factor of more than $15$.
Moreover, all three top-$p$ approaches are significantly faster than inference in the original HMM and are comparable fast.
For the Uniform HMM, top-$0.9$ is around $3.39$ times slower than the original HMM.
Clearly, the lower $p$, the faster the inference.
The reason for the slower runtime is the low sparsity: Sparse matrix multiplication only pays off for high sparsity, which increases with higher $p$.
In the LM HMM, top-$p$ is more than $18$ times faster than the original HMM.
Figure~\ref{figure:runtime_exact} shows the experiments for the three tested HMMs.

\begin{figure}[tbp]
     \centering
     \begin{subfigure}[b]{0.6\textwidth}
         \centering
         \includegraphics[width=\textwidth]{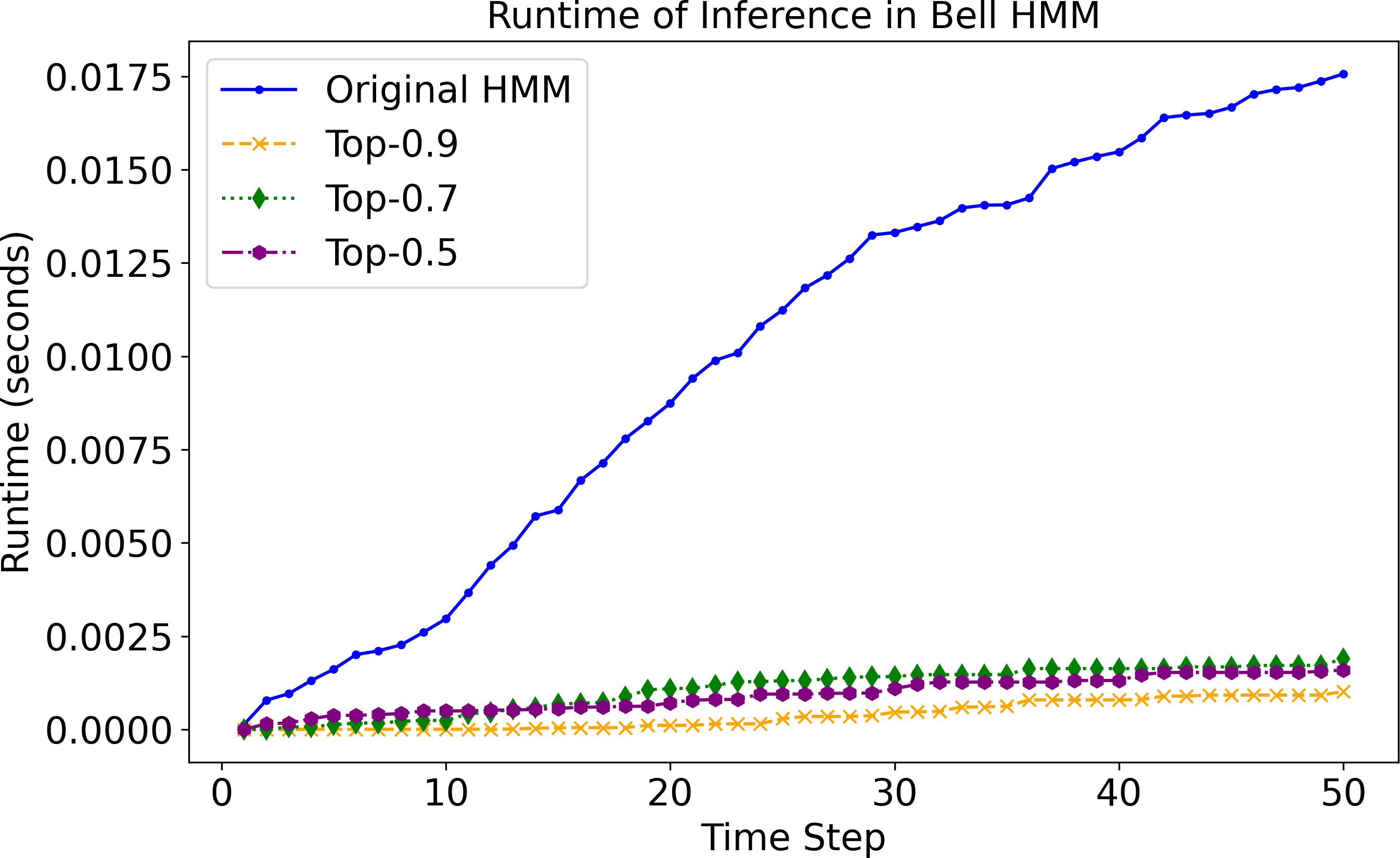}
         \caption{Bell HMM}
         \label{figure:bell_runtime_exact}
     \end{subfigure}
     \hfill
     \begin{subfigure}[b]{0.6\textwidth}
         \centering
         \includegraphics[width=\textwidth]{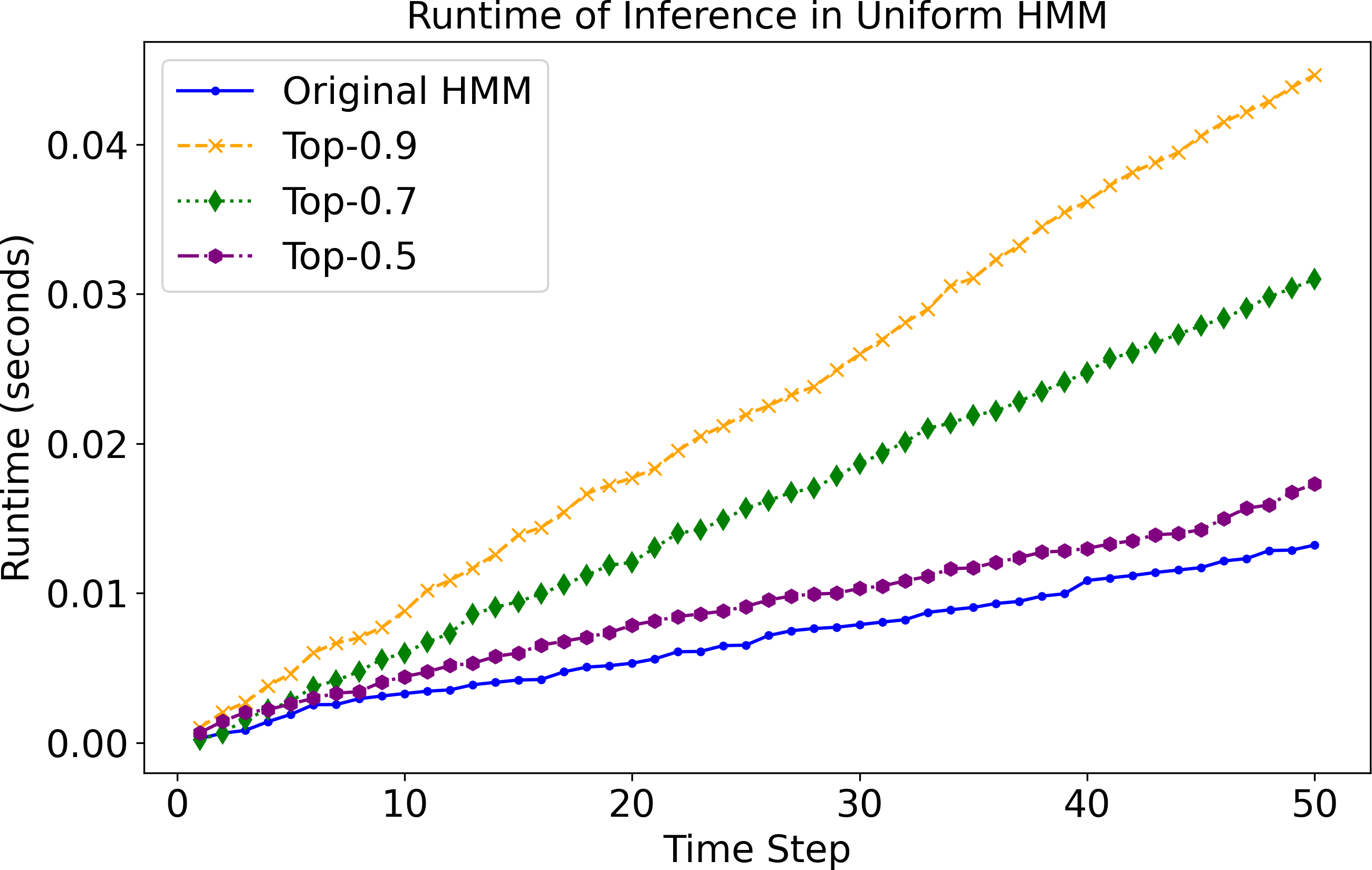}
         \caption{Uniform HMM}
         \label{figure:uniform_runtime_exact}
     \end{subfigure}
     \hfill
     \begin{subfigure}[b]{0.6\textwidth}
         \centering
         \includegraphics[width=\textwidth]{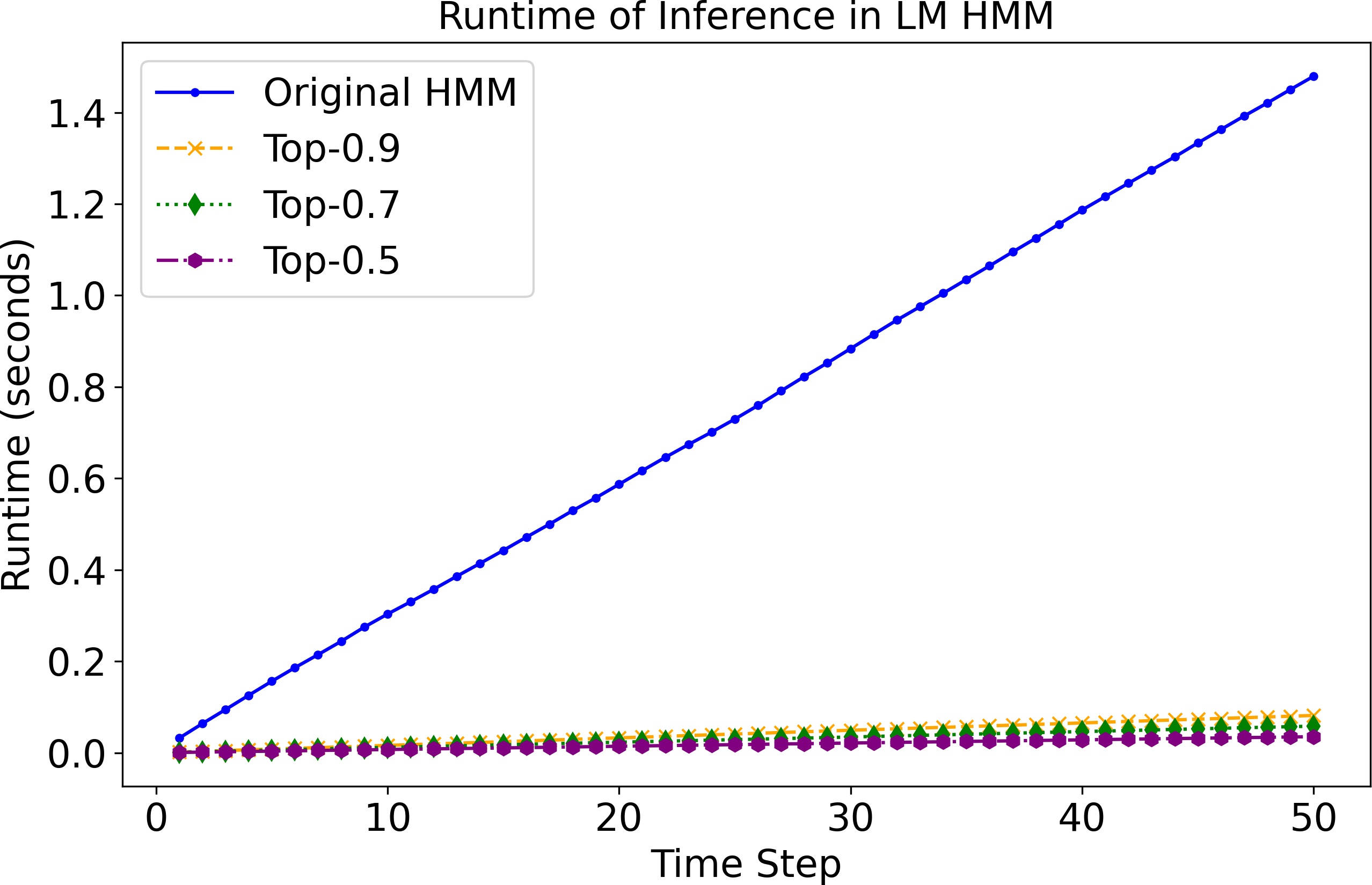}
         \caption{LM HMM}
         \label{figure:lm_runtime_exact}
     \end{subfigure}
     \hfill
     \caption{Runtime for inference in the test HMMs. The blue line shows the runtime in the original HMM and the other three ones for different top-$p$ ones applying HMM$_{off}$.}
     \label{figure:runtime_exact}
\end{figure}

\paragraph{Runtime for Filtering}
Before, we have reported runtimes without observations.
Now, we enter an observation every five time steps.
We omit the LM HMM, since observation in text generation is not useful.
In the Bell HMM, the top-$0.9$ approach is more than $77$ times faster than filtering in the original HMM.
In the Uniform HMM, in contrast to regular prediction, the top-$0.9$ approach is now more than $24$ times faster than the original HMM, due to more probabilistic calculations helping the overhead induced by sparse matrices to pay off.
Figure~\ref{figure:runtime_filtering} shows the experiment for the two tested HMMs.

\begin{figure}[tbp]
     \centering
     \begin{subfigure}[b]{0.8\textwidth}
         \centering
         \includegraphics[width=\textwidth]{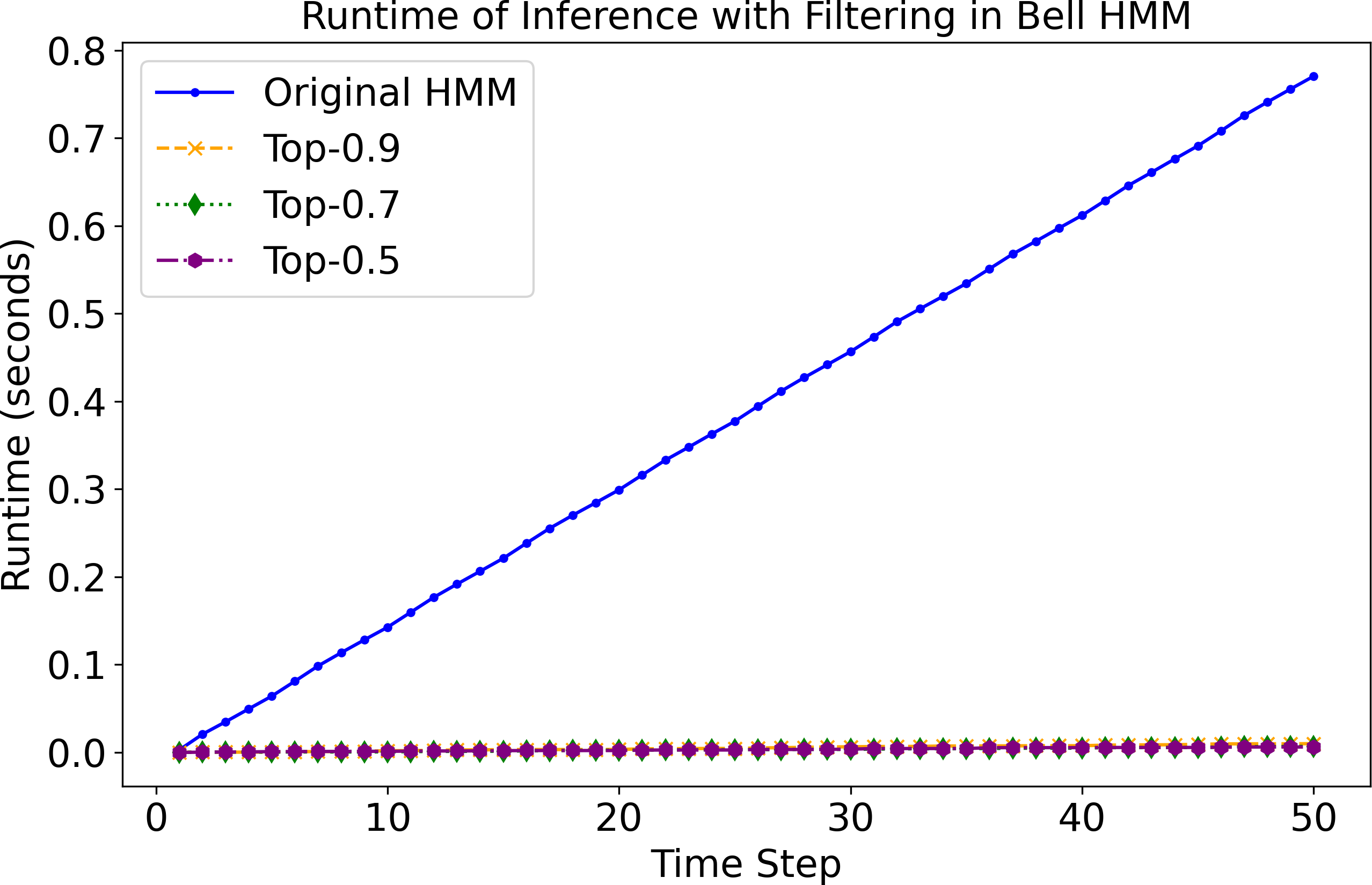}
         \caption{Bell HMM with filtering}
         \label{figure:bell_runtime_filtering}
     \end{subfigure}
     \hfill
     \begin{subfigure}[b]{0.8\textwidth}
         \centering
         \includegraphics[width=\textwidth]{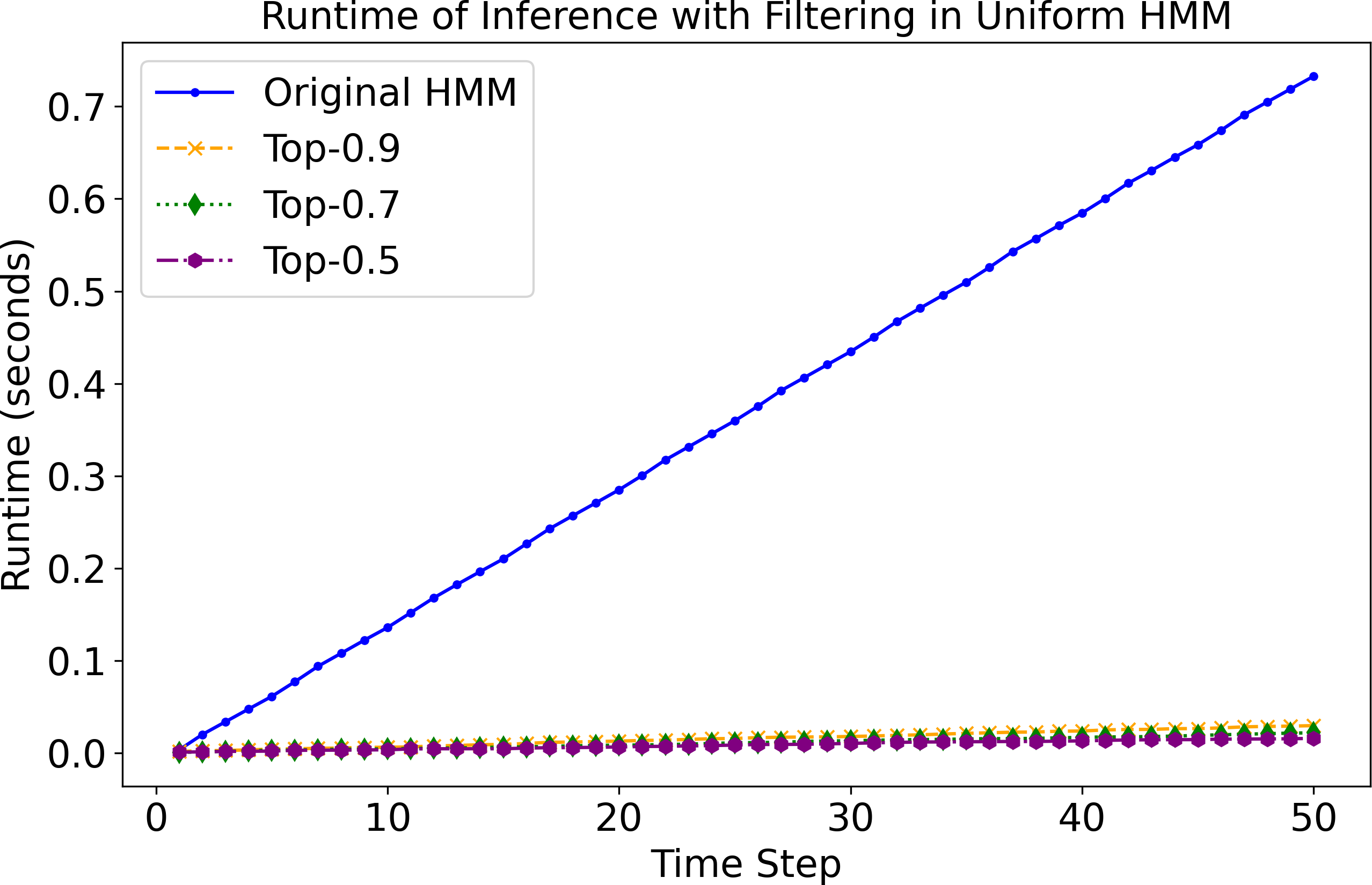}
         \caption{Uniform HMM with filtering}
         \label{figure:uniform_runtime_filtering}
     \end{subfigure}
     \caption{Runtime for filtering in two test HMMs. The blue line shows the runtime in the original HMM and the other three ones for different top-$p$ ones applying HMM$_{off}$.}
     \label{figure:runtime_filtering}
\end{figure}

\paragraph{Total Variation Distance}
The total variation in the Bell HMM is $0.086$ for top-$0.9$, $0.285$ for top-$0.7$ and $0.78$ for top-$0.5$.
In the Uniform HMM, the total variation distance is $0.099$, $0.299$ and $0.499$ for top-$0.9$, top-$0.7$ and top-$0.5$, respectively.
In the LM HMM, the total variation distance is $0.059$, $0.172$ and $0.3$ for top-$0.9$, top-$0.7$ and top-$0.5$, respectively.

\paragraph{Total Variation Distance for Filtering}
With observations, the total variation in the Bell HMM is $0.095$ for top-$0.9$, $0.491$ for top-$0.7$ and $0.846$ for top-$0.5$.
The increased error compared to without filtering is because more approximations are made in the process.
In the Uniform HMM, the total variation distance is the same as for no observations, since the model is uniform.
We again omit the LM HMM, since observation in text generation is not useful.

Regarding which $p$ to choose, our evaluation suggests to start with higher values like $0.9$ and adapt them based on the requirements for runtime and total variation.
Summing up, the top-$p$ approach using HMM$_{off}$ significantly increases sparsity to more than $0.9$ for top-$0.9$.
Moreover, HMM$_{off}$ is at least an order of magnitude faster than the original HMMs, while the total variation distance for top-$0.9$ stays below $0.09$.

%% file: conclusion.tex
\section{Conclusion}

Inference in large (dynamic) probabilistic models like HMMs, DBNs, or GPT-based language models like Llama 3, is a complex task involving expensive operations.
In particular, the whole state space needs to be enumerated for advancing in time.
We propose to only use the top-$p$ transitions in our HMM$_{off}$ approach, i.e., the most probable transitions per state with an accumulated probability of $p$.
Using only the top-$p$ transitions and setting the probability for all other ones to zero, we can significantly speed up inference by at least an order of a magnitude.
Moreover, the error introduced by HMM$_{off}$ can be bound in terms of total variation distance.
For a simple language model, the total variation distance is below $0.059$ for top-$0.9$.
Using only the top-$p$ events in the current state distribution is possible, too, with our HMM$_{on}$ approach.
For HMM$_{on}$, we can show the same error bound as for HMM$_{off}$.
However, the runtime of HMM$_{on}$ is slower than inference in the original HMM, because we need to compute the top-$p$ distribution of the current state distribution in each time step.
The computation of the top-$p$ distribution and their usage in matrix-vector multiplication appears to be slower than regular NumPy matrix-vector multiplication in the original HMM.
Furthermore, by checking the error bound, an agent can decide whether to use the top-$p$ approach. %
In future work, the top-$p$ approach can be utilized for first-order inference in DBNs.
Moreover, the obtained sparse HMM$_{off}$ can be used in edge computing, where a full HMM may be too big to run on.